%% file: neurips_2026.tex
\documentclass{article}

\usepackage[preprint]{neurips_2026}
\usepackage{amsfonts}
\usepackage{graphicx}
\usepackage{bm}

\usepackage[utf8]{inputenc} 
\usepackage[T1]{fontenc}    
\usepackage{hyperref}       
\usepackage{url}            
\usepackage{booktabs}       
\usepackage{amsfonts}       
\usepackage{nicefrac}       
\usepackage{microtype}      
\usepackage{xcolor}         
\usepackage{amsmath}
\usepackage{amsthm}
\usepackage{comment}
\usepackage{algorithm}
\usepackage{cleveref}
\usepackage{algorithmic}  
\usepackage{xcolor}
\usepackage{multirow}
\usepackage{booktabs}

\newcommand{\secondbest}{\rlap{\textsuperscript{*}}}

\newtheorem{theorem}{Theorem}[section]
\newtheorem{proposition}{Proposition}[section]
\newtheorem{lemma}{Lemma}[section]
\newtheorem{definition}{Definition}[section]

\DeclareMathOperator{\Law}{Law}
\DeclareMathOperator*{\argmin}{arg\,min}

\DeclareMathOperator{\Log}{Log}

\title{ReSFlow: Rectifying Categorical Flows on Statistical Manifolds for One-Step Generation}

\author{%
  \normalfont
  \begin{tabular*}{\textwidth}[t]{@{\extracolsep{\fill}}ccc@{}}
    \textbf{Xiaohuan Jia}\textsuperscript{1}
    & \textbf{Renzhe Xu}\textsuperscript{2}
    & \textbf{Xiao Wang}\textsuperscript{2}
    \\[0.25em]
    {\normalfont\normalsize\texttt{xiaohuanjia@mail.imu.edu.cn}}
    & {\normalfont\normalsize\texttt{xurenzhe@sufe.edu.cn}}
    & {\normalfont\normalsize\texttt{wangxiao@sufe.edu.cn}}
    \\[1.5em]
    \multicolumn{3}{c}{%
      \begin{tabular*}{0.70\textwidth}{@{\extracolsep{\fill}}cc@{}}
        \textbf{Jiayun Wu}\textsuperscript{3}
        & \textbf{Shaohua Fan}\textsuperscript{1}\thanks{Corresponding author.}
        \\[0.25em]
        {\normalfont\normalsize\texttt{jiayunw@cmu.edu}}
        & {\normalfont\normalsize\texttt{shfan@imu.edu.cn}}
      \end{tabular*}}
    \\[0.65em]
    \multicolumn{3}{c}{\normalfont
      \textsuperscript{1}Inner Mongolia University;\quad
      \textsuperscript{2}Shanghai University of Finance and Economics;}
    \\[0.15em]
    \multicolumn{3}{c}{\normalfont
      \textsuperscript{3}Carnegie Mellon University}
  \end{tabular*}%
}

\begin{document}

\maketitle

\begin{abstract}
Generative modeling for discrete data has been significantly advanced by learning continuous vector fields on the statistical manifolds to transport simple priors to target discrete distributions. However, because these continuous-time flows are not constrained to follow geodesics on the statistical manifold, thereby requiring multi-step numerical integration for simulation and leading to substantial computational costs during inference. In this work, drawing inspiration from the ``straight and fast'' properties of Rectified Flow in continuous Euclidean space, we propose the Rectified Statistical Flow (ReSFlow) model, which enables one-step categorical generation on statistical manifolds by explicitly learning geodesic trajectories that adhere to the intrinsic Riemannian geometry. By rectifying the flow on the statistical manifold, we theoretically guarantee that trajectories between prior and target distributions approximate geodesics. This rectification significantly reduces the transport cost by aligning the learned trajectories directly along geodesics on the statistical manifold, allowing the model to map pure noise to target discrete data in a single step. Through extensive experiments on diverse real-world generative tasks, ranging from image and text to biological domains, we demonstrate ReSFlow's superior one-step generation quality, which truly embodies geodesically optimal and rapid categorical generation. Our code is available at \url{https://github.com/xiaohuanjia/ReSFlow}.
\end{abstract}

\section{Introduction}

Generative modeling has advanced rapidly, with diffusion and flow-based methods achieving state-of-the-art performance across continuous domains such as image synthesis~\cite{dao2023flow,lipman2022flow}, molecule and protein design~\cite{Campbell2024,zaghen2025riemannian,huguet2024sequence}, and sequence generation~\cite{davis2024fisherflowmatchinggenerative,shaul2024flow}. A key idea is to learn continuous-time dynamics, especially through flow matching, which formulates ODEs with continuous vector fields and provides a simulation-free, theoretically grounded training framework~\cite{albergo2023building}. However, extending these methods to discrete data—common in NLP, drug discovery, and protein engineering—remains difficult, since discrete variables do not have the differentiable structure required by continuous flows. Most existing flow-based generative models for discrete data act directly on the native discrete space by optimizing the categorical probabilities of tokens~\citep{gat2024discrete, liu2024discrete, yoo2025redi, stark2024dirichlet}. Rather than operating strictly within this discrete realm, a promising direction maps discrete variables to statistical manifolds, treating each categorical distribution as a point on the probability simplex with Riemannian geometry~\cite{cheng2024categorical}. This allows smooth vector fields and continuous-time flows to transport priors to target discrete distributions, bridging continuous generative modeling and discrete data synthesis.

However, a critical limitation persists: The learned ODE trajectories deviate from geodesics because the conditional vector fields arising from independent couplings of prior and target samples yield a mixture of intersecting geodesics that, once marginalized, produce highly curved aggregate flows not aligned with any single geodesic path~\cite{cheng2025riemannian}. Simulating these tortuous trajectories demands multi-step numerical integration with fine temporal discretization, incurring substantial computational overhead during inference. This stands in stark contrast to the desiderata of modern generative systems, where rapid, ideally single-step, sampling is essential for scalable deployment. The fundamental tension between the geometric complexity of learned flows and the practical need for efficient generation thus constitutes a major bottleneck for categorical generative models.

In the continuous Euclidean setting, an elegant solution to this problem has emerged through the concept of Rectified Flow~\cite{liu2023flow}. Rectified Flow enforces straight-line trajectories between paired noise and data samples, yielding flows that are both geometrically simple and computationally efficient to simulate. The resulting ``straight and fast'' property enables high-quality generation with very few or even a single integration step, dramatically accelerating inference. Naturally, one might ask: \textit{can we bring the same benefits of trajectory rectification to categorical generation on statistical manifolds?}

The answer, however, is far from straightforward. The probability simplex, when equipped with the Fisher-Rao metric~\cite{Rao1992}, forms a Riemannian manifold whose geometry is fundamentally non-Euclidean. Straight paths in ambient Euclidean space do not respect the intrinsic curvature of this manifold—they may leave the simplex, violate probability constraints, or deviate substantially from the true shortest paths. In this curved geometric setting, the natural analogue of a ``straight path'' is a geodesic: the length-minimizing curve that faithfully adheres to the manifold's intrinsic Riemannian structure. Consequently, achieving efficient one-step categorical generation requires not Euclidean straightening, but rather a fundamentally geometric approach that aligns learned trajectories with the geodesics of the underlying statistical manifold.

In this work, we propose Rectified Statistical Flow (ReSFlow), a framework that enables one-step discrete data generation by rectifying flows on statistical manifolds. The core idea is to map data from the probability simplex onto the positive orthant of the hypersphere via their natural diffeomorphism, and then perform flow rectification in this spherical space where geodesics admit closed-form great-circle expressions. By iteratively recoupling the learned ODE marginals and retraining the vector field on the hypersphere, ReSFlow progressively straightens transport paths toward geodesics, reducing both transport cost and path curvature.   Extensive evaluations across vision, NLP, and bioinformatics benchmarks demonstrate that ReSFlow, requiring only a single sampling step, matches or outperforms state-of-the-art methods that rely on hundreds of steps.

\section{Preliminaries}
\paragraph{Notations.}
Let $\mathcal{X} = \{1,\dots,K\}$ be a finite categorical state space and
$\mathcal{P}(\mathcal{X})$ the probability simplex over $\mathcal{X}$,
i.e.\ the set of all categorical distributions.
We use $\mu, \nu \in \mathcal{P}(\mathcal{X})$ to denote individual
categorical distributions, and
$p, q\in \mathcal{P}(\mathcal{P}(\mathcal{X}))$ to denote distributions
over categorical distributions.
We let $p_0, q \in \mathcal{P}(\mathcal{P}(\mathcal{X}))$ denote
the source (prior) and target (data) distributions, respectively.
We write $\Delta^{K-1}$ for the $(K{-}1)$-simplex and
$\mathbb{S}^{K-1}_+$ for the positive orthant of the unit hypersphere in
$\mathbb{R}^K$.
The square-root map $\pi\colon \Delta^{K-1}\to\mathbb{S}^{K-1}_+$,
defined by $\pi(\mu)_i = \sqrt{\mu_i}$, is a diffeomorphism that renders
the Fisher--Rao metric on $\Delta^{K-1}$ isometric to the round metric on
$\mathbb{S}^{K-1}_+$;
we overload notation and write $\pi(p)$ for the pushforward of a
distribution $p \in \mathcal{P}(\mathcal{P}(\mathcal{X}))$ onto
$\mathbb{S}^{K-1}_+$.

\subsection{Statistical Flow Matching}
\label{sec::sfm}
Statistical Flow Matching (SFM)~\citep{cheng2024categorical} performs flow matching for discrete categorical data directly on the statistical manifold induced by the Fisher information metric.
Since the Fisher inner product is ill-defined on the boundary of
$\Delta^{K-1}$, SFM operates on $\mathbb{S}^{K-1}_+$ via the square-root
map~$\pi$ (see Appendix \ref{app:statistical_manifold} for detailed motivations).
Specifically, we independently draw $\mu_0 \sim p_0$ and
$\mu_1 \sim q$ and map them to
$X_0 = \pi(\mu_0)$, $X_1 = \pi(\mu_1)$ on $\mathbb{S}^{K-1}_+$.
SFM constructs a conditional path via geodesic interpolation:
\begin{equation}\label{eq:geodesic}
  X_t := \psi_t(X_0, X_1)
       = \exp_{X_0}\!\bigl(t\,\log_{X_0} X_1\bigr),
  \quad t\in[0,1],
\end{equation}
where $\exp$ and $\log$ are the Riemannian exponential and logarithmic maps
on $\mathbb{S}^{K-1}$.  Letting
$\theta = \arccos(X_0^\top X_1)$ denote the geodesic distance, the
conditional vector field $u_t = \tfrac{\mathrm{d}}{\mathrm{d}t}X_t$ is
\begin{equation}\label{eq:cond-vf}
  u_t^{\,\mathbb{S}}(X_t \mid X_0, X_1)
  = \frac{\theta}{\sin\theta}
    \bigl[-\cos\bigl((1{-}t)\theta\bigr)\,X_0
           +\cos(t\theta)\,X_1\bigr].
\end{equation}
A parameterised velocity field $v_\theta(X_t,t)$ is trained by minimizing:
\begin{equation}\label{eq:sfm-loss}
  \mathcal{L}_{\mathrm{SFM}}
  = \mathbb{E}_{t,\,X_0,\,X_1}
    \bigl[\lVert u_t^{\,\mathbb{S}}(X_t\mid X_0,X_1)
                 - v_\theta(X_t,t)\rVert_2^2\bigr],
\end{equation}
where $t\sim U[0,1]$, $X_0\sim\pi(p_0)$, $X_1\sim\pi(q)$.
By exploiting the intrinsic geometry of $\mathbb{S}^{K-1}_+$, SFM ensures
the flow remains on the simplex and provides a principled optimal transport
path, avoiding quantisation issues prevalent in Euclidean treatments of
categorical data.

\subsection{Rectified Flow}\label{sec:rf}
Rectified Flow~\citep{liu2023flow} learns straight ODE trajectories that
transport a source distribution $\rho_0$ to a target distribution $\rho_1$,
with an iterative \emph{reflow} procedure that progressively straightens
the paths to enable few-step or even one-step generation.
Originally formulated in $\mathbb{R}^d$, we present it here in our unified
notation; in Section~\ref{sec::method} we extend it to
$\mathbb{S}^{K-1}_+$.
Given a coupling $(\mu_0, \mu_1)$ with $\mu_0 \sim p_0$ and
$\mu_1 \sim q$, Rectified Flow defines the linear interpolation path
$X_t = t\,\mu_1 + (1{-}t)\,\mu_0$ for $t\in[0,1]$ and learns a velocity
model $v_\theta$ by minimizing
\begin{equation}\label{eq:rf-obj}
  \hat\theta
  = \argmin_\theta\;
    \mathbb{E}\bigl[\lVert \mu_1 - \mu_0
                     - v_\theta(X_t,\,t)\rVert_2^2\bigr],
  \quad t\sim\mathrm{Uniform}([0,1]).
\end{equation}
Once trained, new samples are generated by solving the ODE
$\mathrm{d}Z_t = v_{\hat\theta}(Z_t,t)\,\mathrm{d}t$ forward from
$Z_0 \sim p_0$ to obtain $Z_1$, yielding a full trajectory
$\mathbf{Z}=\{Z_t : t\in[0,1]\}$.
A distinguishing feature of Rectified Flow is the \emph{reflow} operation.
Starting from an initial independent coupling
$(Z_0^0, Z_1^0) = (\mu_0, \mu_1)$, the $k$-th reflow step applies the
above procedure to the induced coupling $(Z_0^k, Z_1^k)$, whose endpoints
are obtained by simulating the previously learned flow: $(Z_0^{k+1}, Z_1^{k+1}) = \texttt{RectFlow}\bigl((Z_0^k, Z_1^k)\bigr).$

Each reflow iteration reduces transport cost and trajectory curvature,
bringing the induced coupling closer to a deterministic map. In the limit,
$v_\theta(Z_t,t)$ is nearly constant along each path and the ODE can be
solved in a single Euler step.
Although reflow successfully straightens trajectories in Euclidean space, it
remains unexplored how to extend it to non-Euclidean manifolds such as
$\mathbb{S}^{K-1}_+$, where straight lines must be replaced by geodesics. More discussion on related work is in Appendix \ref{app::Related Work}.

\section{Method}
\label{sec::method}
As discussed above, the extension of rectified flow to statistical manifolds poses additional challenges, as the model needs to be fully aware of the intrinsic geometry induced by the Fisher information metric to properly define and straighten the ODE trajectories. In particular, the velocity field estimation and the evolution of marginal distributions—both central to the rectified flow framework—must be reformulated to respect the curvature and geometric structure of the statistical manifold. To address these challenges, we propose the Rectified Categorical Flow model, which reflows and straightens the ODE trajectories along geodesics on a mapped hyperspherical manifold.
\begin{figure*}[t]
  \centering
  \includegraphics[width=0.8\textwidth]{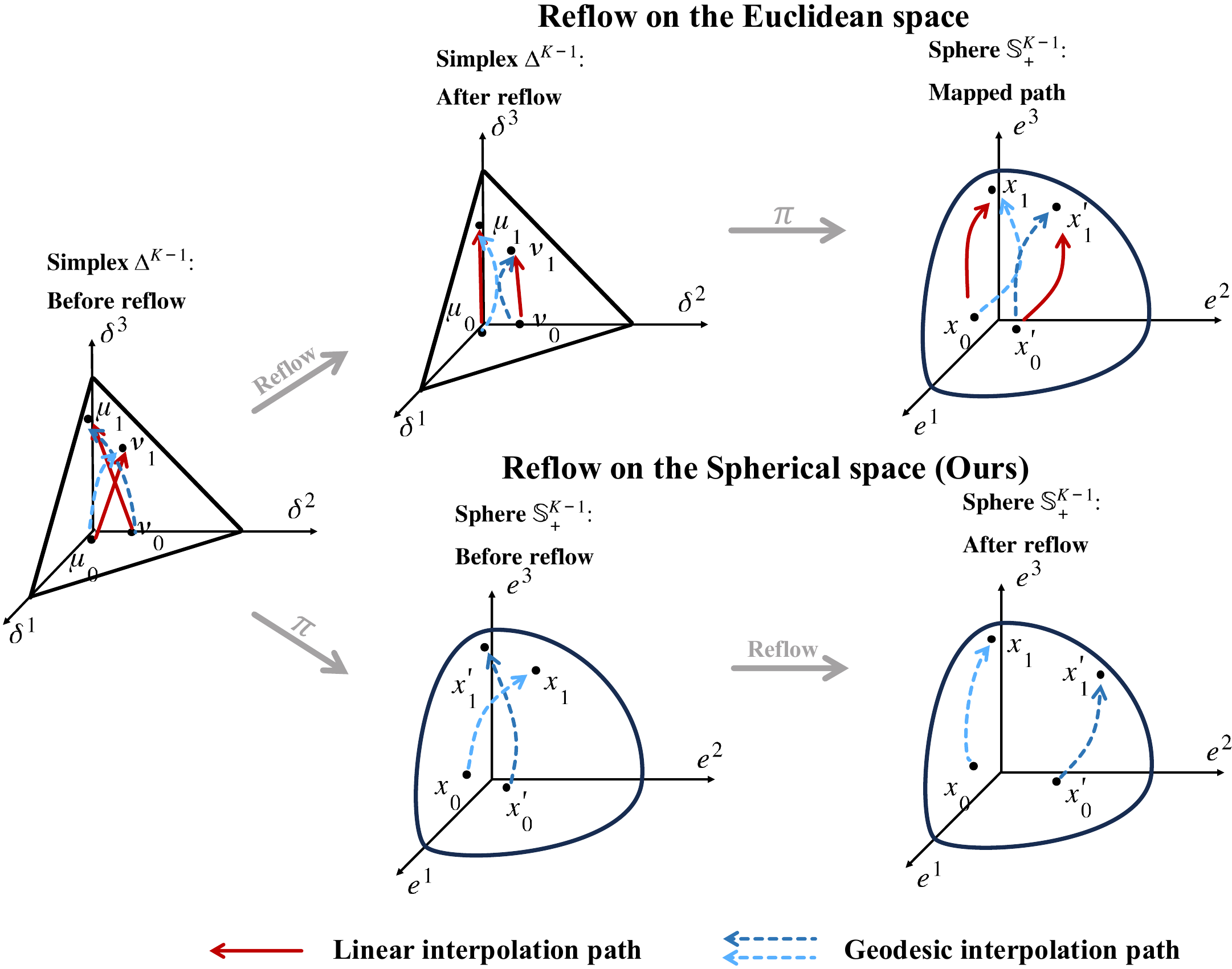}
  \caption{\textbf{Overview of Spherical Rectified Flow.}
  \textit{Top:} Standard reflow straightens linear paths on the simplex, but their corresponding spherical geodesics still intersect, preventing one-step sampling.
\textit{Bottom:} We perform reflow directly on the sphere $\mathbb{S}^{K-1}_+$. This straightens the geodesics themselves to be non-intersecting, ensuring faithful one-step generation.}
  \label{fig:ReSFlow_overview}
  \vspace{-10pt}
\end{figure*}

\subsection{Why Euclidean Reflow Fails on the Statistical Manifold}\label{sec:why-fail}

Before presenting our approach, we highlight why a naïve application of
Euclidean rectified flow to categorical distributions is fundamentally
inadequate. Recall that standard rectified flow defines straight-line
interpolations $X_t = t\,\mu_1 + (1{-}t)\,\mu_0$ in $\mathbb{R}^d$ and
trains the velocity field to match the constant direction
$\mu_1 - \mu_0$. Reflow guarantees that iterative re-coupling yields non-crossing straight paths, enabling single-step generation.  However, when the data live on the probability
simplex $\Delta^{K-1}$ equipped with the Fisher--Rao metric, this
guarantee breaks down:

\begin{proposition}[Non-preservation of non-crossing under Fisher--Rao geodesics]\label{prop:crossing}
Let\, $\Delta^{K-1} = \{\mu\in\mathbb{R}^K : \mu_i > 0,\;\sum_i \mu_i = 1\}$
be the open probability simplex with Fisher--Rao metric
$g_{ij}(\mu) = \delta_{ij}/\mu_i$.
For two pairs $(\mu_0, \mu_1)$ and $(\nu_0, \nu_1)$ in $\Delta^{K-1}$, let
\begin{equation}
  \ell_\mu(t) = (1-t)\,\mu_0 + t\,\mu_1, \quad
  \ell_\nu(t) = (1-t)\,\nu_0 + t\,\nu_1, \quad t\in[0,1]
\end{equation}
be the linear interpolations, and $\gamma_\mu(t)$, $\gamma_\nu(t)$ the
Fisher--Rao geodesics.  Then for $K \ge 3$,
\begin{equation}
  \ell_\mu([0,1]) \cap \ell_\nu([0,1]) = \emptyset
  \quad\not\Rightarrow\quad
  \gamma_\mu([0,1]) \cap \gamma_\nu([0,1]) = \emptyset.
\end{equation}
\end{proposition}

\begin{proof}[Proof sketch]
Under the square-root embedding
$\pi\colon \mu \mapsto (\sqrt{\mu_1},\dots,\sqrt{\mu_K}) \in
\mathbb{S}^{K-1}_+$, Fisher--Rao geodesics become great-circle arcs:
$\pi(\gamma_\mu(t))
= \frac{\sin((1-t)\beta_\mu)}{\sin\beta_\mu}\,\pi(\mu_0)
 +\frac{\sin(t\,\beta_\mu)}{\sin\beta_\mu}\,\pi(\mu_1)$,
where $\beta_\mu = \arccos(\sum_i \sqrt{\mu_{0,i}\,\mu_{1,i}})$.
The nonlinear map $\pi$ does not preserve the non-crossing property:
geodesics curve toward the interior of $\Delta^{K-1}$ relative to linear
paths, creating intersections even when the straight paths are separated.\hfill
\end{proof}

Proposition~\ref{prop:crossing} reveals that Euclidean reflow straightens
paths in $\mathbb{R}^K$ so that linear interpolation path no longer cross, yet the
induced geodesics on the statistical manifold can still intersect
(Fig.~\ref{fig:ReSFlow_overview}, top).  This motivates our approach: we perform
reflow directly on $\mathbb{S}^{K-1}_+$, straightening the geodesic paths
themselves to obtain non-crossing trajectories that enable single-step
sampling (Fig.~\ref{fig:ReSFlow_overview}, bottom).

\subsection{Geodesic Interpolation on the Hypersphere}

Leveraging the square-root diffeomorphism $\pi: \Delta^{K-1} \to \mathbb{S}^{K-1}_+$, we lift all operations to the positive orthant of the unit hypersphere, where the Fisher-Rao geometry becomes isometric to the standard spherical geometry. On $\mathbb{S}^{K-1}_+$, geodesics are great-circle arcs, and the geodesic interpolation between a coupling $(X_0, X_1)$ on  sphere  $\mathbb{S}^{K-1}_+$ is given by the spherical linear interpolation:
\begin{equation}
X_t = \text{slerp}(X_0, X_1; t) = \frac{\sin((1-t)\theta)}{\sin(\theta)} X_0 + \frac{\sin(t\theta)}{\sin(\theta)} X_1, \quad t \in [0,1], 
\end{equation}
where $\theta = \arccos(X_0^\top X_1)$ is the geodesic distance between $X_0$ and $X_1$. This interpolation has several desirable properties: (i) $X_t$ remains on $\mathbb{S}^{K-1}_+$ for all $t \in [0,1]$ by construction, automatically satisfying the simplex constraints after the inverse map $\pi^{-1}$; (ii) the curve $t \mapsto X_t$ is a constant-speed geodesic with $\|\dot{X}_t\| = \theta$ for all $t$; and (iii) the tangent vector along this path has the closed-form expression given in \Cref{eq:cond-vf}.

Our key observation is that, just as Euclidean rectified flow uses linear interpolation $tX_1 + (1-t)X_0$ as its ``straight'' path, the natural equivalent on the hypersphere is the geodesic \textit{slerp}. Consequently, a flow with geodesic integral curves maintains a constant velocity magnitude, enabling exact simulation in a single Euler step.

\subsection{Rectified Categorical Flow via Spherical Reflow}
\label{sec:spherical_reflow}

We now present the core of our method: the iterative reflow procedure on $\mathbb{S}_+^{K-1}$ that progressively straightens the learned ODE trajectories toward geodesics. The procedure consists of two alternating stages, followed by an efficient inference phase.

\textbf{Stage 1: Initial Flow Training ($k=0$).} We begin with an independent coupling $(X_0, X_1) \sim \pi(p_0) \times \pi(q)$. We train a velocity field $v_\theta$ by minimizing the spherical flow matching objective:
\begin{equation}
    \label{eq:ReSFlow_initial}
    \mathcal{L}_{\text{ReSFlow}}^{(0)} = \mathbb{E}_{t, X_0, X_1} \left[ \left\| v_\theta(X_t, t) - u_t^S(X_t \mid X_0, X_1) \right\|_2^2 \right],
\end{equation}
where $X_t = \mathrm{slerp}(X_0, X_1; t)$ is the geodesic interpolant. Crucially, $v_\theta(X_t, t)$ is constrained to the tangent space $\mathcal{T}_{X_t}\mathbb{S}^{K-1}$ to ensure the flow remains on the manifold (details in Appendix~\ref{sec::Tangent space projection}). 

At this stage, the marginal velocity field $v_\theta$ approximates the conditional expectation $\mathbb{E}[u_t^S \mid X_t]$. Because the independent coupling produces a mixture of intersecting geodesics, averaging their tangent vectors at intersection points yields a direction that deviates from any single geodesic. Consequently, the resulting ODE trajectories are curved, requiring many steps to simulate accurately.

\textbf{Stage 2: Trajectory-Induced Recoupling (Reflow).} To straighten the trajectories, we exploit the uniqueness of ODE solutions. Given the trained model $v_{\hat{\theta}}^{(k)}$, we simulate the ODE forward on $\mathbb{S}_+^{K-1}$:
\begin{equation}
    \label{eq:ReSFlow_ode_sim}
    \frac{dX_t}{dt} = v_{\hat{\theta}}^{(k)}(X_t, t), \quad X_0 \sim \pi(p_0), \quad t \in [0, 1].
\end{equation}
This simulation produces new endpoint pairs $(Z_0^{k+1}, Z_1^{k+1})$. We then train the $(k+1)$-th rectified flow using the geodesic interpolation between these recoupled pairs:
\begin{equation}
    \label{eq:ReSFlow_reflow}
    \mathcal{L}_{\text{ReSFlow}}^{(k+1)} = \mathbb{E}_{t, Z_0^{k+1}, Z_1^{k+1}} \left[ \left\| v_\theta^{(k+1)}(Z_t^{k+1}, t) - u_t^S(Z_t^{k+1} \mid Z_0^{k+1}, Z_1^{k+1}) \right\|_2^2 \right].
\end{equation}

The deterministic ODE defines an injective transport map. Because this map is continuous, nearby points in the source distribution are mapped to nearby points in the target. Thus, the geodesics connecting the recoupled pairs tend to run parallel rather than intersect. With strictly fewer crossing geodesics, the conditional expectation at each point is dominated by a single geodesic direction, making the updated velocity field closer to a pure geodesic field. The complete reflow iteration on $\mathbb{S}_+^{K-1}$ can be summarized as: $(Z_0^{k+1}, Z_1^{k+1}) = \texttt{RectFlow}_{\mathbb{S}}\bigl((Z_0^k, Z_1^k)\bigr).$

\textbf{Inference: Efficient One-Step Generation.} 
Once the iterative reflow procedure converges, the learned vector field $v_\theta$ effectively models non-intersecting, constant-speed geodesics. Given a prior sample $X_0 \sim \pi(p_0)$, we can project the initial velocity onto the manifold to generate the target data sample $X_1$ in a single step via the spherical exponential map:
\begin{equation}
    \label{eq:one_step_gen}
    \hat{X_1} = \mathrm{Exp}_{X_0}(v_\theta(X_0, 0)) = X_0 \cos(\|v_0\|_2) + \frac{v_0}{\|v_0\|_2} \sin(\|v_0\|_2),
\end{equation}
where $v_0 = v_\theta(X_0, 0)$. This formulation bridges continuous-time ODEs with discrete geometric operations, unlocking highly accurate one-step generation for categorical data. The final discrete sample is obtained by applying the inverse square-root map: $\hat{p} = \pi^{-1}(\hat{X}_1) = (\hat{X}_{1,1}^2, \ldots, \hat{X}_{1,K}^2) \in \Delta^{K-1}$, followed by taking the argmax (or sampling from) this categorical distribution. We summarize the complete Rectified Categorical Flow pipeline in Algorithm~\ref{alg:ReSFlow}.

\subsection{Theoretical Guarantees of Spherical Reflow}
\label{sec:theory}

To ground our geometric intuition, we establish formal theoretical guarantees for the spherical reflow process based on the framework of non-Euclidean probability transport. Detailed proofs for all statements are provided in Appendix~\ref{app:spherical-proofs}.

First, a valid generative model must ensure that modifying the trajectories does not alter the target data distribution. We confirm that spherical reflow strictly preserves the marginal distributions.

\begin{theorem}[Spherical Marginal Preservation; Informal; See~\Cref{thm:app-marginal-preservation}]
\label{thm:marginal}
Let $X_t$ be the spherical geodesic interpolation process. If the flow is rectifiable, its induced spherical rectified flow $Z_t$ preserves the marginals: $\mathrm{Law}(Z_t) = \mathrm{Law}(X_t)$ for all $t \in [0, 1]$.
\end{theorem}

To formalize trajectory straightness, we define the Spherical Transport Cost for a coupling $(X_0, X_1)$ on $\mathbb{S}_+^{K-1}$ as $\mathbb{E}\!\left[d_{\mathbb{S}}(X_0,X_1)^2\right] = \mathbb{E}[\arccos(X_0^\top X_1)^2]$, which measures the expected squared geodesic distance. Reflow theoretically guarantees the optimization of this transport cost.

\begin{theorem}[Contraction of Spherical Transport Cost; Informal; See~\Cref{thm:app-distance-cost}]
\label{thm:cost_contraction}
Let $(Z_0^k, Z_1^k)$ be the coupling induced by the $k$-th reflow iteration. The spherical transport cost is non-increasing:
\begin{equation}
    \label{eq:cost_contraction}
    \mathbb{E}\!\left[d_{\mathbb{S}}(Z_0,Z_1)^2\right]
    \le
    \mathbb{E}\!\left[d_{\mathbb{S}}(X_0,X_1)^2\right].
\end{equation}
\end{theorem}
\begin{proof}
Proof sketch in Appendix~\ref{sec::Spherical Trans}: The ODE trajectories have finite length, and the geodesic connecting the recoupled endpoints represents the absolute shortest path on the manifold, enforcing a monotonic decrease in transport cost. 
\end{proof}

Finally, we provide the following theorem shows that iterative spherical rectification progressively straightens trajectories at a guaranteed rate:

\begin{theorem}[Spherical Reflow Straightening Rate; Informal; See~\Cref{thm:app-reflow-rate}]
\label{thm:straightening_rate}
Let $(X_0, X_1)$ be a coupling on $\mathbb{S}_+^{K-1}$ with $\mathbb{E}[d_{\mathbb{S}}(X_0,X_1)^2]<+\infty$. Set $(Z_0^0,Z_1^0)=(X_0,X_1)$ and $Z^{k+1}=\texttt{RectFlow}_{\mathbb{S}}((Z_0^k,Z_1^k))$ for $k=0,\ldots,N_{\mathrm{reflow}}-1$. Then
\begin{equation}
    \sum_{k=0}^{N_{\mathrm{reflow}}-1}\left[S_{\mathbb{S}}(Z^{k+1})+V_{\mathbb{S}}((Z_0^k,Z_1^k))\right]\le\mathbb{E}\!\left[d_{\mathbb{S}}(X_0,X_1)^2\right],
\end{equation}
where $S_{\mathbb{S}}\!\ge 0$ measures the excess kinetic energy over the geodesic cost and $V_{\mathbb{S}}\!\ge 0$ measures the velocity variance. Consequently,
\begin{equation}
    \min_{0\le k\le N_{\mathrm{reflow}}-1}\left[S_{\mathbb{S}}(Z^{k+1})+V_{\mathbb{S}}((Z_0^k,Z_1^k))\right]=\mathcal{O}(1/N_{\mathrm{reflow}}).
\end{equation}
See Appendix~\ref{app::Straightening} for the detailed proof.
\end{theorem}

This theoretical framework confirms that combining hyperspherical geometry with iterative reflow explicitly forces the learned vector fields to approximate constant-speed geodesics, providing the rigorous foundation for our highly accurate one-step generation via the exponential map. Theoretical results are empirically validated in Appendix \ref{app::addi}.

\section{Experiments}
\subsection{Experimental setup}

Our proposed ReSFlow models arbitrary distributions over categorical probability measures and enables one-step sampling via geodesic rectification on the statistical manifold. We first validate this property on a synthetic Swiss-roll distribution on the simplex, and then evaluate ReSFlow on real-world discrete generation tasks, including binarized MNIST image generation, character-level Text8 modeling, and conditional promoter DNA sequence design. Our primary comparison is with SFM~\cite{cheng2024categorical}, which achieves state-of-the-art performance on these benchmarks. The additional baselines reported in that work rely on multi-step sampling and do not outperform SFM; we therefore do not repeat them here. To disentangle the effects of geometry and rectification, we compare four controlled variants: \emph{Linear Flow Matching} (LinearFM)~\cite{cheng2024categorical}, which uses Euclidean linear interpolation on the simplex; \emph{Statistical Flow Matching} (SFM), which maps the simplex to the positive orthant of the hypersphere and follows spherical geodesics; \emph{Rectified Linear Flow Matching} (R-LFM), which applies rectification in the Euclidean setting; and \emph{Rectified Statistical Flow} (ReSFlow), our proposed spherical reflow model. This $2 \times 2$ design separates the choice of geometry, Euclidean versus spherical, from the effect of rectification. Moreover, we also compare with SOTA flow-based one-step generation ReDi~\cite{yoo2025redi}, which rectifies the flow trajectories of continuous token embeddings in Euclidean space. The fundamental difference between their setting and ours lies in the data preprocessing requirement. While their method relies on an external pre-trained model to tokenize the input data, our ReSFlow operates as a fully end-to-end framework directly on the raw data space. By bypassing the discrete tokenization step, our approach is highly advantageous for domains and modalities where well-established pre-trained tokenizers are unavailable. Due to the absence of well-established tokenizers for binarized MNIST and promoter DNA sequences, we restrict our comparison with ReDi solely to the Text8, where characters are tokenized using the GPT-2 tokenizer.

We evaluate all methods under different numbers of function evaluations (NFE). For sample-quality metrics such as FID~\cite{heusel2017gans} and SP-MSE, we report both one-step generation with NFE$=1$ and accurate ODE simulation using the adaptive Dormand--Prince 5 (Dopri5) solver~\cite{RKfamily}. For the Swiss-roll benchmark, we report only Dopri5-based results, including ODE-based NLL and Dopri5 sampling time. Likelihood-based metrics are evaluated only with ODE integration: NLL estimation uses the continuous-time change-of-variables formula and therefore requires accurate integration of the learned dynamics, while a single Euler or exponential-map step does not yield reliable likelihoods. For Text8, BPC is computed from the variational integral formulation rather than generated samples, and is therefore not tied to one-step sampling. Thus, one-step sampling measures fast generation quality, whereas ODE-based integration evaluates NLL and BPC. We also report wall-clock sampling time for binarized MNIST, promoter sequence design, and Swiss-roll to quantify practical runtime , but omit timing for Text8, since the reported BPC metric is based on likelihood evaluation rather than sample generation. Detailed architectures and hyperparameters are provided in Appendix~\ref{app:implementation}.

\vspace{-5pt}

\subsection{Toy Example: Swiss Roll on the Simplex}

We first evaluate ReSFlow on a Swiss-roll distribution on the $2$-simplex as a low-dimensional geometric sanity check. Following Cheng et al.~\cite{cheng2024categorical}, we normalize a two-dimensional Swiss roll and embed it into the simplex by setting $\mu_3=1-\mu_1-\mu_2$. All methods use the same time-dependent MLP and training budget. LinearFM and R-LFM use Euclidean interpolation on the simplex; SFM and ReSFlow use the square-root map with spherical geodesic interpolation on $\mathbb{S}^{2}_{+}$. This controlled comparison separates the effect of geometry from the effect of rectification. As shown in Table~\ref{tab:swissroll_results}, ReSFlow achieves an NLL of $-2.25$, improving over LinearFM ($-1.88$) and R-LFM ($-1.87$), and remains competitive with SFM ($-2.38$). It also attains the fastest ODE sampling time among all methods, reducing the runtime from $0.0679$s for SFM to $0.0361$s. This indicates that ReSFlow preserves the geometric advantage of Fisher--Rao geodesics and enabling efficient ODE generation.

The samples in Fig.~\ref{fig:swissroll_samples} further illustrate the role of spherical reflow under one-step sampling. 
Non-rectified LinearFM and SFM fail to recover the Swiss-roll structure in the one-step setting, because their marginal vector fields average over many crossing conditional paths, producing curved trajectories that are poorly approximated by a single update. 
Euclidean reflow improves one-step behavior relative to LinearFM, but it still rectifies paths under the wrong geometry, leading to samples that only partially capture the target structure. 
ReSFlow instead performs reflow directly on the hypersphere, aligning the learned trajectories with Fisher--Rao geodesics and producing the most faithful one-step samples. 
This supports the central claim that geometry and rectification are complementary: the hyperspherical formulation provides the correct transport geometry, while reflow reduces trajectory curvature for accurate one-step sampling.

\begin{table}[t]
\centering
\caption{Swiss-roll results on the 2-simplex. We report ODE-based NLL and ODE sampling time for generating 1,000 samples. Bold indicates the best result, and `$*$' indicates the second-best result.}
\label{tab:swissroll_results}
\small
\begin{tabular}{lllll}
\toprule
Metric & LinearFM & SFM~\cite{cheng2024categorical} & R-LFM & ReSFlow \\
\midrule
NLL $\downarrow$
& -1.88 \,\textpm\, 0.06
& \textbf{-2.38 \,\textpm\, 0.08}
& -1.87 \,\textpm\, 0.06
& -2.25 \,\textpm\, 0.08\secondbest \\
Time (s) $\downarrow$
& 0.1059
& 0.0679
& 0.0512\secondbest
& \textbf{0.0361} \\
\bottomrule
\end{tabular}
\vspace{-1em}
\end{table}

\begin{figure*}[t]
  \centering
  \includegraphics[width=0.9\textwidth,keepaspectratio]{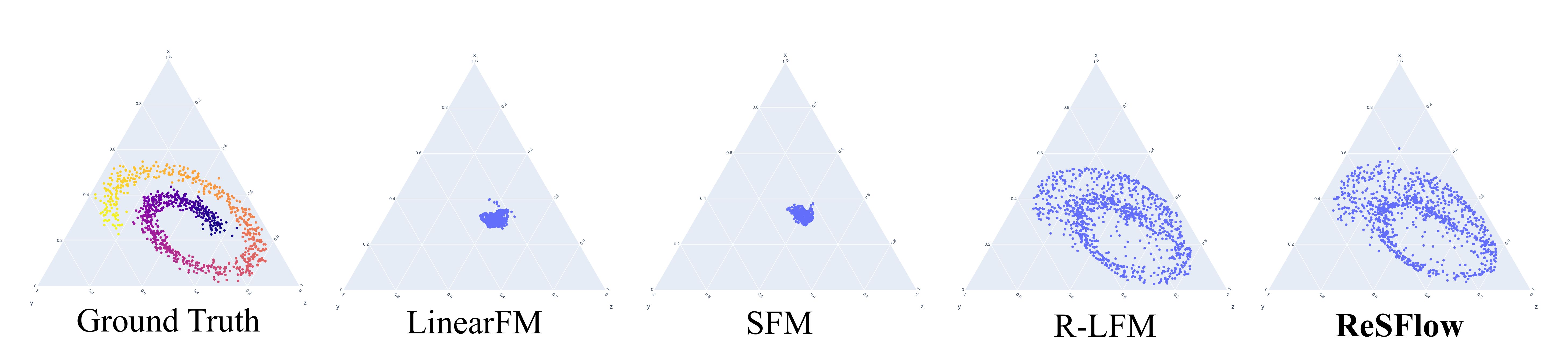}
    \caption{Swiss-roll samples on the simplex under one-step Euler sampling. }
  \label{fig:swissroll_samples}
\end{figure*}

\subsection{Binarized MNIST}

We next evaluate on binarized MNIST to test whether the geometric and rectification effects observed in the Swiss-roll experiment extend to high-dimensional discrete generation. Each image is modeled as a collection of Bernoulli categorical variables, so good samples require both local binary consistency and global digit structure. This benchmark therefore allows us to examine two questions in a realistic setting: whether Fisher--Rao geometry improves categorical transport, and whether reflow makes the learned dynamics accurate under NFE = 1 sampling. Following the evaluation protocol above, we compare adaptive ODE sampling with Euler one-step generation.

Table~\ref{tab:bmnist_results} reports mean FID and wall-clock sampling time over three random seeds. Under adaptive ODE sampling, SFM achieves the best FID, confirming that spherical geodesic interpolation provides a stronger reference geometry than Euclidean interpolation for binary categorical data. ReSFlow is slightly worse than SFM in this full-ODE regime, but remains competitive and is the fastest ODE-based method. Under one-step sampling, non-rectified LinearFM and SFM degrade substantially, indicating that their learned trajectories are too curved to be approximated by a single update. Rectification alleviates this issue, and ReSFlow achieves the best one-step FID by performing reflow under the hyperspherical Fisher--Rao geometry. Compared with SFM under accurate ODE sampling, one-step ReSFlow achieves an orders-of-magnitude improvement in sampling speed ($0.452$ vs.\ $112.81$ seconds per 1,000 samples), while incurring only a mild and acceptable degradation in FID ($5.973$ vs.\ $4.748$). These results show that ReSFlow improves trajectory geometry and enables efficient one-step sampling.

Figure~\ref{fig:bmnist_nfe} further supports the benefit of reflow under the Fisher--Rao geometry. ReSFlow achieves a mean FID of $5.973\pm0.442$ with a single function evaluation, while SFM requires 34 evaluations to obtain a comparable result ($5.922\pm0.262$). This reduction in sampling steps shows that rectification makes the learned dynamics more suitable for one-step generation.

\begin{table}[t]
\centering
\caption{Binarized MNIST results. Mean $\pm$ standard deviation over three random seeds; time is for generating 1,000 samples. Bold indicates the best mean, and $*$ indicates the second-best mean.}
\label{tab:bmnist_results}
\footnotesize
\setlength{\tabcolsep}{3pt}
\begin{tabular}{lllll}
\toprule
\multicolumn{1}{c}{\multirow[c]{2}{*}[-0.5ex]{Method}}
& \multicolumn{2}{c}{ODE (NFE = adaptive)}
& \multicolumn{2}{c}{Euler one-step (NFE = 1)} \\
\cmidrule(lr){2-3} \cmidrule(lr){4-5}
& FID $\downarrow$ 
& Time / 1k samples (s) $\downarrow$
& FID $\downarrow$
& Time / 1k samples (s) $\downarrow$ \\
\midrule
LinearFM~\cite{cheng2024categorical}
& 5.187\,\textpm\,0.129\secondbest
& 87.13\,\textpm\,0.44
& 25.317\,\textpm\,0.517
& \textbf{0.406\,\textpm\,0.178} \\
SFM~\cite{cheng2024categorical}
& \textbf{4.748\,\textpm\,0.092}
& 112.81\,\textpm\,8.20
& 35.640\,\textpm\,0.487
& 0.719\,\textpm\,0.175 \\
R-LFM
& 5.681\,\textpm\,0.252
& 61.19\,\textpm\,0.28\secondbest
& 6.261\,\textpm\,0.379\secondbest
& 0.583\,\textpm\,0.244 \\
ReSFlow
& 5.541\,\textpm\,0.252
& \textbf{36.07\,\textpm\,0.60}
& \textbf{5.973\,\textpm\,0.442}
& 0.452\,\textpm\,0.243\secondbest \\
\bottomrule
\end{tabular}
\end{table}

\subsection{Text8}

We further evaluate ReSFlow on character-level language modeling with Text8 \cite{mahoney2011large}, where each token is a categorical variable and the sequence distribution is substantially more structured than the toy and image settings. This experiment is designed to test whether the proposed manifold-aware rectification remains effective for high-entropy discrete sequences with a larger alphabet and stronger contextual dependencies. Since text generation is typically evaluated by predictive coding efficiency, we report bits per character (BPC), which naturally reflects whether the learned categorical transport is efficient and sufficiently well aligned with the data distribution: a lower BPC indicates that the model assigns higher probability to the observed character sequences and induces less distorted transport over categorical states. Since BPC is evaluated from the continuous-time likelihood, it requires ODE integration; we therefore report only ODE-based BPC in Table~\ref{tab:text8_results}.

Table~\ref{tab:text8_results} shows that the choice of geometry has a pronounced effect on categorical sequence modeling. SFM substantially improves over LinearFM, reducing BPC from $2.154$ to $1.830$, which supports the importance of modeling categorical data under the Fisher--Rao geometry rather than using Euclidean interpolation on the simplex. Moreover, R-LFM performs worse than LinearFM, increasing BPC to $3.267$. This agrees with the failure mode discussed in Section~\ref{sec:why-fail}: rectifying paths in the ambient Euclidean space optimizes the wrong notion of straightness and transport cost. The effect is especially harmful for character-level sequences, where local geometric errors can accumulate over many categorical variables and time steps, ultimately degrading the likelihood. Furthermore, although ReDi leverages pre-trained token embeddings to map discrete data into a continuous space, its Euclidean rectification only yields a BPC of $2.306$, exposing a substantial performance gap compared to our ReSFlow ($1.261$). This sharp contrast underscores that relying on pre-trained continuous representations cannot compensate for geometric mismatches, proving the necessity of optimizing discrete data strictly within its suitable intrinsic space. By performing rectification under the correct manifold geometry, ReSFlow achieves the best overall BPC, confirming the effectiveness of our approach.
Compared with binarized MNIST, Text8 has a larger categorical alphabet and richer sequential dependencies, so accumulated geometric errors are more likely to affect the final likelihood. By combining hyperspherical statistical-manifold modeling with reflow, ReSFlow preserves the intrinsic geometric advantage of SFM while further aligning the learned trajectories with geodesic transport on the manifold. These results support our design choice: the manifold formulation captures the proper geometry of discrete data, and rectification should be performed within that geometry to obtain efficient and reliable categorical generation.
\begin{table}[htbp]
\centering
\caption{Text8 character-level modeling results. ODE-based BPC, reported as mean $\pm$ standard deviation over three random seeds. Bold indicates the best mean, and $*$ indicates the second-best mean.}
\label{tab:text8_results}
\begin{tabular}{ll}
\toprule
Method & BPC $\downarrow$ (adaptive ODE) \\
\midrule
LinearFM~\cite{cheng2024categorical} & 2.15448\,\textpm\,0.00024 \\
SFM~\cite{cheng2024categorical} & 1.82998\,\textpm\,0.00028\secondbest \\
R-LFM & 3.26687\,\textpm\,0.00012 \\
ReDi~\cite{yoo2025redi} & 2.30562\,\textpm\,0.01010 \\
ReSFlow & \textbf{1.26146\,\textpm\,0.00036} \\
\bottomrule
\end{tabular}

\end{table}
\subsection{Promoter Sequence Design}

We finally evaluate on conditional promoter DNA sequence design \cite{avdeyev2023dirichlet}, a long-range biological sequence generation task over the nucleotide alphabet $\{\mathrm{A},\mathrm{C},\mathrm{G},\mathrm{T}\}$. This experiment is designed to test whether the proposed manifold-aware rectification remains effective when the data are both discrete and conditionally structured: unlike binarized images, promoter sequences require preserving position-dependent motifs and global consistency with the transcription initiation signal. Thus, this benchmark stresses the same issue discussed in the previous sections, namely whether a learned categorical flow can be compressed to one step without losing the geometry needed for faithful conditional generation. We use signal proxy mean squared error (SP-MSE) as the evaluation metric, where lower is better. Following the standard protocol, generated promoters are passed to a pre-trained Sei predictor, and SP-MSE measures the discrepancy between the predicted activity profiles of generated and real human promoter sequences. 

We report results under both accurate Dopri5 ODE simulation and one-step sampling in Table~\ref{tab:promoter_results}. 
Under ODE sampling, ReSFlow is nearly twice as fast as LinearFM ($67.94$ vs.\ $119.04$ seconds per 1,000 sequences), while yielding a slightly higher mean SP-MSE ($0.028743$ vs.\ $0.027734$).
More importantly, ReSFlow preserves its performance in the one-step setting: its mean SP-MSE is $0.028718$, compared with $0.028743$ under ODE sampling, while the sampling time decreases from $67.94$ seconds to $0.409$ seconds. 
By comparison, SFM performs competitively with full ODE integration but degrades to $0.035564$ under one-step sampling, showing that manifold-aware interpolation alone is insufficient for reliable single-step generation. 
R-LFM gives substantially worse SP-MSE values, $0.052146$ with ODE and $0.060504$ with Euler one-step sampling, suggesting that rectifying trajectories in the ambient Euclidean geometry can be detrimental for structured categorical sequences. 
These results support the central design of ReSFlow: reflow should be performed under the intrinsic Fisher--Rao geometry so that the resulting trajectories remain suitable for efficient one-step conditional generation.

\begin{table}[htbp]
\centering
\caption{Conditional promoter design results. Mean $\pm$ standard deviation over three random seeds; time is for generating 1,000 sequences of length 1024. Bold indicates the best mean, and $*$ indicates the second-best mean.}
\label{tab:promoter_results}
\footnotesize
\setlength{\tabcolsep}{3pt}
\begin{tabular}{lllll}
\toprule
\multicolumn{1}{c}{\multirow[c]{2}{*}[-0.5ex]{Method}}
& \multicolumn{2}{c}{ODE (NFE=adaptive)}
& \multicolumn{2}{c}{Euler one-step (NFE=1)} \\
\cmidrule(lr){2-3} \cmidrule(lr){4-5}
& SP-MSE $\downarrow$
& Time  $\downarrow$
& SP-MSE $\downarrow$
& Time  $\downarrow$ \\
\midrule
LinearFM~\cite{cheng2024categorical}
& \textbf{0.027734\,\textpm\,0.000897}
& 119.04\,\textpm\,16.57
& 0.029449\,\textpm\,0.001673\secondbest
& 0.410\,\textpm\,0.004\secondbest \\
SFM~\cite{cheng2024categorical}
& 0.028015\,\textpm\,0.000257\secondbest
& 82.18\,\textpm\,4.26
& 0.035564\,\textpm\,0.007768
& 0.410\,\textpm\,0.005 \\
R-LFM
& 0.052146\,\textpm\,0.000781
& \textbf{67.88\,\textpm\,4.17}
& 0.060504\,\textpm\,0.000351
& \textbf{0.409\,\textpm\,0.001} \\
ReSFlow
& 0.028743\,\textpm\,0.000755
& 67.94\,\textpm\,0.35\secondbest
& \textbf{0.028718\,\textpm\,0.000366}
& \textbf{0.409\,\textpm\,0.003} \\
\bottomrule
\end{tabular}
\vspace{-1em}
\end{table}

\section{Conclusion}
We proposed Rectified Statistical Flow (ReSFlow), a principled generative framework for discrete data that performs flow matching directly on the probability simplex under its natural spherical geometry. By proposing a reflow procedure in the spherical space via the square-root parameterization, ReSFlow respects the intrinsic Riemannian structure of categorical distributions while admitting closed-form conditional velocity fields.  We further established a theoretical guarantee on the straightening rate of iterative spherical rectification, and demonstrated the effectiveness of ReSFlow over existing baselines across diverse discrete generation tasks. As future work, the straightening property of ReSFlow naturally lends itself to single-step generation on other spherical data, such as unit-norm embeddings, directional statistics, and rotation-valued representations.

\begin{ack}
We thank the anonymous reviewers for their time, thoughtful comments, and constructive suggestions, which helped improve the clarity and presentation of this work.

This work was supported by the National Natural Science Foundation of China (Grant Nos. 62402263 and 62672259), the Major Science and Technology Special Project Program of the Inner Mongolia Autonomous Region (Grant No. 2026ZD0089), and the Open Competition Project under the Science and Technology Breakthrough Initiative of the Inner Mongolia Autonomous Region (Grant No. 2026KJTW0003). This work was also supported by the Inner Mongolia University High-level Talent Project (Grant No. 10000-A260015/322). Renzhe Xu’s research was supported by the Shanghai Sailing Program (No. 24YF2711600).
\end{ack}

\bibliographystyle{unsrt}
\bibliography{references}

\clearpage
\appendix

\section{Statistical Manifold of Categorical Distributions and the Spherical Projection}
\label{app:statistical_manifold}

To formulate continuous-time generative dynamics for discrete data, a natural approach is to consider the $n$-class categorical distributions as points lying on the $(n-1)$-dimensional probability simplex $\Delta^{n-1}$. Geometrically, $\Delta^{n-1}$ can be treated as a Riemannian manifold, known as the statistical manifold, where the intrinsic geometry is governed by the Fisher Information Metric. 

For a distribution $\mu \in \Delta^{n-1}$ (where $\sum_{i=1}^n \mu_i = 1$ and $\mu_i \geq 0$) and tangent vectors $u, v$ at $\mu$, the Fisher inner product is defined as:
\begin{equation}
    \langle u, v \rangle_{\mu} = \sum_{i=1}^n \frac{u_i v_i}{\mu_i}.
    \label{eq::inner_Fisher}
\end{equation}

\paragraph{The Boundary Illness.} 
While the Fisher metric elegantly captures the information geometry of probability distributions, \Cref{eq::inner_Fisher} reveals a critical flaw for generative modeling: the inner product becomes ill-defined on the boundary of the simplex. As any probability mass $\mu_i \to 0$, the denominator causes the metric to diverge. In discrete generation tasks, the target data points are typically deterministic, represented as one-hot vectors. These one-hot vectors lie exactly on the sharp corners (boundaries) of the simplex. Consequently, performing flow matching or integrating vector fields directly on $\Delta^{n-1}$ using the Fisher metric leads to severe numerical instability and singularities near the targets.

\paragraph{Necessity of the Spherical Projection.}
To circumvent this boundary illness while preserving the intrinsic geometry, we introduce the square-root diffeomorphism:
\begin{equation}
    \pi: \Delta^{n-1} \to \mathbb{S}^{n-1}_+, \quad \mu_i \mapsto x_i = \sqrt{\mu_i}.
\end{equation}
Since $\sum_{i=1}^n \mu_i = 1$, it naturally follows that $\sum_{i=1}^n x_i^2 = 1$. This transformation smoothly maps the original probability simplex to the positive orthant of the unit sphere $\mathbb{S}^{n-1}_+$. Under this mapping, the intrinsic geometry of the statistical manifold is beautifully preserved up to a constant factor. Specifically, the geodesic distance $d_{\text{cat}}$ on the simplex is strictly proportional to the standard great-circle distance $d_S$ on the sphere:
\begin{equation}
    d_{\text{cat}}(\mu, \nu) = 2 d_\mathbb{S}(\pi(\mu), \pi(\nu)),
\end{equation}
where $d_\mathbb{S}(x, y) = \arccos(\sum_{i=1}^n x_i y_i)$ for $x, y \in \mathbb{S}^{n-1}_+$. This indicates that we can seamlessly perform geometric operations on the unit sphere instead of the simplex.

Beyond preserving the distance metric, this transformation fundamentally resolves the boundary illness. For a point $x \in \mathbb{S}^{n-1}_+$, its tangent space is given by $T_x\mathbb{S}^{n-1}_+ = \{ \tilde{u} \in \mathbb{R}^n \mid \sum_{i=1}^n x_i \tilde{u}_i = 0 \}$. For any two tangent vectors $\tilde{u}, \tilde{v} \in T_x\mathbb{S}^{n-1}_+$, the inner product on the sphere simplifies to the standard Euclidean dot product:
\begin{equation}
    \langle \tilde{u}, \tilde{v} \rangle_{x} = \sum_{i=1}^n \tilde{u}_i \tilde{v}_i.
    \label{Eq::inner_S}
\end{equation}
Contrasting \Cref{Eq::inner_S} with \Cref{eq::inner_Fisher} makes the advantage immediately clear: the spherical inner product does not involve the coordinates ($x_i$ or $\mu_i$) in the denominator. Consequently, the metric, vector fields, and geodesics remain perfectly well-defined and numerically stable everywhere. By formulating our continuous dynamics on $\mathbb{S}^{n-1}_+$, our framework elegantly resolves the singularity at the boundary ($x_i = 0$), enabling stable and exact flow matching to deterministic one-hot targets.

\section{Related Work}
\label{app::Related Work}
Flow matching was originally proposed for continuous domains as an alternative to diffusion, learning transport along a prescribed probability path \citep{lipman2023flow,albergo2023building,liu2023flow}. Extending this idea to discrete spaces is more challenging, since the dynamics evolve over categorical distributions rather than Euclidean variables. Related work has modeled discrete denoising in continuous time using continuous-time Markov chains (CTMCs), bridging diffusion-style methods and discrete data \citep{campbell2024generative}. Building on this perspective, Discrete Flow Matching \citep{gat2024discrete} extends flow matching to discrete generation, while subsequent methods further refine the geometry and path design, including Fisher Flow Matching \citep{davis2024fisher} and Dirichlet Flow Matching \citep{stark2024dirichlet}. Compared with discrete diffusion, which typically learns reverse transitions, these approaches emphasize direct probability transport. Among them, Statistical Flow Matching (SFM) \citep{cheng2024categorical} is particularly appealing, as it formulates categorical generation on the simplex as a statistical manifold, better respecting the geometry of categorical distributions while avoiding explicit CTMC jump-rate parameterization. This provides a cleaner and often more stable framework for discrete generation, and thus serves as the foundation of our model.

Flow-based models can generate high-quality discrete samples, but inference is often costly, typically requiring tens to hundreds of function evaluations. Moreover, most one- or few-step generative methods focus on continuous data \citep{liu2023flow,salimans2022progressive,cheng2025riemannian,yin2024one}; in discrete spaces, fast sampling has been explored mainly within diffusion-based frameworks through distillation or specialized acceleration schemes \citep{gulcehre2025beyond,hayakawa2024distillation,liu2024discrete,xu2024energy}. Flow matching is often considered more suitable for low-step generation because it learns transport dynamics rather than local denoising and allows the design of straighter probability paths, e.g., rectified or OT-inspired flows. Its deterministic ODE formulation may therefore be easier to compress into a few updates, although standard flow matching is not inherently one-step. Recently, ReDi~\citep{yoo2025redi} extended rectified flow to discrete spaces via iterative reflow. However, this approach—like many continuous relaxation methods for discrete spaces—heavily requires pre-trained token embeddings to function, restricting its applicability when training purely from scratch. Moreover, the extension is performed entirely in Euclidean space, ignoring the intrinsic Riemannian geometry of the probability simplex. As a result, linear interpolation paths are not true geodesics on the statistical manifold, leading to suboptimal transport costs and trajectories that are inherently harder to straighten through reflow. Consequently, ReDi's practical advantages over standard discrete flow matching remain less clear-cut than those of rectified flow in continuous settings.

\section{Implementation Details}
\subsection{Tangent Space Projection}
\label{sec::Tangent space projection}
The learned $v_\theta(X_t, t)$ must lie in the tangent space $T_{X_t}\mathbb{S}^{K-1}$ at each point to ensure the flow remains on the manifold. We enforce this via tangent space projection:

\begin{equation}
 v_\theta^{\text{tan}}(X_t, t) = v_\theta(X_t, t) - \langle v_\theta(X_t, t), X_t \rangle\, X_t,    
\end{equation}

which removes the normal component of the predicted velocity at each point on the sphere.
\subsection{Algorithm Procedure}

Algorithm~\ref{alg:ReSFlow} summarizes the complete ReSFlow procedure.
For each reflow iteration $k$, we train a velocity field $v_\theta$ on the
current coupling $(Z_0^k, Z_1^k)$. At each training step, a mini-batch of
$B$ triplets $\{(t_j, z_{0,j}, z_{1,j})\}_{j=1}^{B}$ is sampled, where
$t_j \sim \mathcal{U}[0,1]$ and $(z_{0,j}, z_{1,j})$ are drawn from the
current coupling. Geodesic interpolants $z_{t_j}$ and their corresponding
tangent velocities $u_{t_j}$ are computed via spherical linear interpolation
(slerp) on $\mathbb{S}^{K-1}_+$. The network parameters $\theta$ are then
updated by minimizing the mini-batch Monte Carlo estimate of the projected
velocity matching loss:
$\frac{1}{B}\sum_{j=1}^{B}\|v_\theta(z_{t_j}, t_j)
- \langle v_\theta(z_{t_j}, t_j), z_{t_j}\rangle z_{t_j} - u_{t_j}\|^2$.
After training converges, the coupling is rectified by simulating the learned
ODE forward to produce new pairs $(Z_0^{k+1}, Z_1^{k+1})$, progressively
straightening the flow trajectories. At inference time, a single Euler step
via the exponential map yields the final estimate
$\hat{p} \in \Delta^{K-1}$.

\begin{algorithm}[htbp]
\caption{Rectified Statistical Flow (ReSFlow)}
\label{alg:ReSFlow}
\begin{algorithmic}[1]
\REQUIRE Prior $p_0$, target $q$, diffeomorphism $\pi: \Delta^{K-1} \!\to\! \mathbb{S}^{K-1}_+$, reflow iterations $N_{\text{reflow}}$

\STATE \textbf{Initialize:} Sample $(X_0, X_1) = (\pi(\mu), \pi(\nu))$ independently, $\mu \sim p_0$, $\nu \sim q$
\STATE \hspace*{4.2em} Set $(Z_0^{0}, Z_1^{0})=(X_0, X_1)$

\FOR{$k = 0, 1, \ldots, N_{\text{reflow}}$}
    \STATE \textcolor{gray}{// Train: learn velocity field on current coupling}
    \REPEAT
        \STATE Sample mini-batch $\{(t_j,\, z_{0,j},\, z_{1,j})\}_{j=1}^{B}$, \; $t_j \sim \mathcal{U}[0,1]$, \; $(z_{0,j},\, z_{1,j})$ from coupling $(Z_0^{k}, Z_1^{k})$
        \STATE $z_{t_j} \leftarrow \mathrm{slerp}(z_{0,j}, z_{1,j};\, t_j)$, \; $\theta_j = \arccos(z_{0,j}^\top z_{1,j})$ \hfill $\triangleright$ Geodesic interpolation on $\mathbb{S}^{K-1}_+$
        \STATE $u_{t_j} \leftarrow \frac{\theta_j}{\sin\theta_j}\!\left[-\cos((1\!-\!t_j)\theta_j)\, z_{0,j} + \cos(t_j\theta_j)\, z_{1,j}\right]$
        \STATE Update $\theta$ by minimizing $\frac{1}{B}\sum_{j=1}^{B}\|v_\theta(z_{t_j}, t_j) - \langle v_\theta(z_{t_j},t_j), z_{t_j}\rangle z_{t_j} - u_{t_j}\|^2$
    \UNTIL{converged}
    \IF{$k < N_{\text{reflow}}$}
        \STATE \textcolor{gray}{// Reflow: recouple via ODE simulation}
        \STATE $\forall\, Z_0^{k+1} \!\sim\! \pi(p_0)$: \; $Z_1^{k+1} \leftarrow \texttt{ODESolve}(v_{\hat{\theta}^{(k)}},\, Z_0^{k+1},\, 0 \!\to\! 1)$
    \ENDIF
\ENDFOR

\STATE
\STATE \textbf{One-Step Inference:}
\STATE Sample $X_0 \sim \pi(p_0)$; \; $w \leftarrow v_{\hat{\theta}}(X_0, 0)$
\STATE $\hat{X}_1 \leftarrow \cos(\|w\|)\, X_0 + \sin(\|w\|)\, \frac{w}{\|w\|}$ \hfill $\triangleright$ Exponential map on $\mathbb{S}^{K-1}$
\STATE \textbf{return} $\hat{p} = \pi^{-1}(\hat{X}_1) \in \Delta^{K-1}$
\end{algorithmic}
\end{algorithm}

\input{theories}

\section{Experimental Setup}
\label{app:implementation}

In this appendix, we provide dataset-specific implementation details for all
benchmarks. Since the network architectures and training protocols differ
across datasets, we describe the data processing, model architecture, and
optimization setup separately for each benchmark. Unless otherwise specified,
all rectified variants use a single reflow iteration, and all reflow models are
initialized from the previous checkpoint and fine-tuned rather than trained
from scratch. We summarize sampling procedures and evaluation metrics at the
end of this section.

\subsection{Swiss Roll on the Simplex}
\label{app:setup_swissroll}

\paragraph{Dataset and preprocessing.}
We follow the Swiss-roll-on-simplex construction used in prior categorical
flow-matching work. The data lie on the $2$-simplex with $K=3$ categories,
i.e., each sample is a probability vector $\mu \in \Delta^2$. We generate
$1000$ samples and use the resulting fixed empirical distribution as the
target distribution $q$ for this toy experiment. Since this benchmark is used
as a geometric sanity check, we do not use a separate validation or test split.
Reported likelihoods are evaluated on the same fixed empirical distribution and
should be interpreted as relative comparisons among methods rather than as
generalization metrics.

The construction maps a two-dimensional Swiss roll to the simplex as follows.
Let $x$ and $y$ denote the two Swiss-roll coordinates. We first shift each
coordinate by subtracting its batchwise minimum and adding a padding constant
$\mathrm{pad}=1.0$. We then define
\[
    M = \max_i (x_i + y_i) + \mathrm{pad},
    \qquad
    z_i = M - (x_i + y_i),
\]
and obtain the simplex point by normalizing
\[
    \mu_i = \frac{(x_i, y_i, z_i)}{M}.
\]
This construction ensures that all coordinates are nonnegative and sum to one.
There is no additional dataset-level $\epsilon$-smoothing parameter; numerical
stabilization inside the flow model uses $\epsilon=10^{-4}$ where needed.

The prior distribution $p_0$ is sampled on the simplex by drawing independent
exponential random variables, normalizing them to sum to one, applying a small
clamp for numerical stability, and renormalizing. The target distribution $q$
is the empirical distribution over the $1000$ generated Swiss-roll samples.
Unless otherwise stated, we use random seed $42$ for data and model
initialization, and seed $0$ for generating reflow pairs.

\paragraph{Model architecture.}
For this benchmark, all methods use the same MLP vector-field backbone. The
input has per-position dimension $3$ and tensor shape $(B,1,3)$, and the
network outputs a vector field with the same per-position dimension. The hidden
dimension is set to $128$.

The encoder consists of an input branch, a time branch, and a fusion head. The
input branch maps the current state through two linear layers with ReLU
activation. The time branch applies two linear layers to a sinusoidal
time-step embedding. The time embedding dimension is $128$. The resulting
state and time features are concatenated and passed through a fusion MLP with
three linear layers and ReLU activations to produce the final vector-field
prediction. We do not use dropout or normalization layers in this toy
architecture.

All compared methods use this same MLP backbone in the Swiss-roll experiment.
The methods differ only in their geometry and flow construction. LinearFM and
R-LFM use Euclidean interpolation and projection on the simplex, whereas SFM
and ReSFlow use the square-root map and spherical geometry on
$\mathbb{S}^{K-1}_{+}$. Tangent-space projection is applied as described in
Appendix~\ref{sec::Tangent space projection}.

\paragraph{Training protocol.}
For the initial flow-matching model, we train with Adam using learning rate
$10^{-3}$ and no weight decay. Since the dataset contains only $1000$ points,
we use full-batch training with batch size $1000$. Each model is trained for
$2000$ epochs, corresponding to one parameter update per epoch. We do not use
learning-rate scheduling, warmup, gradient clipping, or exponential moving
average for this toy experiment.

For the reflow variants, we use one reflow iteration. In both R-LFM and
ReSFlow, the reflow model is initialized from the previous checkpoint and then
fine-tuned, rather than trained from scratch. At the beginning of the reflow
stage, we generate $10{,}000$ recoupled endpoint pairs in batches of $256$
using ODE sampling from $t=0$ to $t=1$ with solver tolerances
$\mathrm{atol}=10^{-5}$ and $\mathrm{rtol}=10^{-5}$. The resulting model is
trained for $2000$ epochs using mini-batches of $256$ recoupled pairs. The two
methods differ only in the geometry used for reflow: R-LFM performs
Euclidean/straight-line reflow in the simplex geometry, whereas ReSFlow uses
spherical tangent reflow under the square-root map.

\subsection{Binarized MNIST}
\label{app:setup_bmnist}

\paragraph{Dataset and preprocessing.}
We use the static binarized MNIST dataset with the standard
$50{,}000/10{,}000/10{,}000$ train/validation/test split. Each image has
spatial resolution $28 \times 28$. The original data are represented as
$784$ binary variables, and we treat each pixel as a categorical variable with
$K=2$ categories. Concretely, each binary pixel value $x \in \{0,1\}$ is
represented by the two-dimensional simplex coordinate $(x,1-x)$. Thus, the
model input is arranged as a tensor of shape $(28,28,2)$.

The prior distribution $p_0$ is sampled independently at each pixel location by
drawing exponential random variables for the two categories, normalizing them
onto the simplex, applying a small clamp for numerical stability, and
renormalizing. The data are not smoothed or dequantized. Unless otherwise
specified, numerical stabilization inside the flow model uses
$\epsilon=10^{-4}$.

No additional image preprocessing is applied. In the convolutional vector-field
network, the simplex-valued input is internally rescaled from $[0,1]$ to
$[-1,1]$ before the first convolutional layer.

\paragraph{Model architecture.}
For binarized MNIST, we use a convolutional vector-field predictor. The input
and output tensors both have shape $(B,28,28,2)$, where the last dimension
corresponds to the two categorical states of each pixel. Internally, the
categorical dimension is treated as the channel dimension for convolutional
processing.

All methods compared on binarized MNIST use the same convolutional backbone
unless otherwise stated. The differences between methods are in the geometric
parameterization and the reflow procedure. LinearFM and R-LFM use
Euclidean/simplex geometry, whereas SFM and ReSFlow use the square-root map and
spherical geometry on $\mathbb{S}^{K-1}_{+}$. The shared backbone uses four
residual stages, four top-down refinement blocks, ELU activations,
InstanceNorm++ normalization, and additive sinusoidal time embeddings with
embedding dimension $128$.

\paragraph{Training protocol.}
For the base model, we train with Adam using learning rate
$3 \times 10^{-4}$, weight decay $0$, and batch size $256$. Models are trained
for $100{,}000$ optimization steps. We use a reduce-on-plateau learning-rate
schedule with decay factor $0.5$, patience $10$, and minimum learning rate
$10^{-4}$. No warmup is used. We clip the global gradient norm to $100$.
No exponential moving average is used. Unless otherwise specified, the random
seed is set to $42$.

For rectified models, we generate recoupled endpoint pairs from the previously
trained model by ODE sampling from $t=0$ to $t=1$ with solver tolerances
$\mathrm{atol}=10^{-5}$ and $\mathrm{rtol}=10^{-5}$. The rectified model is
initialized from the previous checkpoint and fine-tuned rather than trained
from scratch. For each method, we generate $50{,}000$ recoupled pairs in
batches of $256$ and train the rectified model for $100{,}000$ optimization
steps with batch size $256$ and learning rate $10^{-4}$. ReSFlow uses
spherical tangent reflow, whereas R-LFM uses Euclidean/simplex reflow.
Experiments were run on an NVIDIA GeForce RTX 5090 GPU with CUDA 13.0 and
PyTorch 2.7.0.

\subsection{Text8}
\label{app:setup_text8}

\paragraph{Dataset and preprocessing.}
We use the Text8 character-level language modeling benchmark. We follow the
standard character-level Text8 split: the first $90$M characters are used for
training, the next $5$M for validation, and the final $5$M for testing. The
vocabulary consists of the $26$ lowercase English letters and the space
character, resulting in $27$ symbols. Each character is mapped to an integer
token.

We use a sequence length of $256$. Each training example is a contiguous window
of $256$ characters. Tokens are represented as one-hot vectors, so each sample
has shape $(256,27)$, and a mini-batch has shape $(B,256,27)$. No additional
text normalization, smoothing, or dequantization is applied beyond the
character-level encoding.

The prior distribution $p_0$ is sampled independently at each sequence position
by drawing exponential random variables over the $27$ categories, normalizing
them onto the simplex, applying a small clamp for numerical stability, and
renormalizing. Unless otherwise specified, numerical stabilization inside the
flow model uses $\epsilon=10^{-4}$.

\paragraph{Model architecture.}
For Text8, all methods using the DiT backbone share the same architecture unless
otherwise specified. The vector-field predictor operates on sequences of length
$256$ with vocabulary size $27$, using $12$ transformer blocks with hidden
dimension $768$, $12$ attention heads, MLP expansion ratio $4$, GELU
feed-forward activations, SiLU time-conditioning activations, dropout $0.1$,
layer normalization, and rotary positional embeddings. Time is encoded by a
sinusoidal timestep embedding of size $256$ followed by a two-layer MLP
producing a $128$-dimensional conditioning vector, which is injected into each
block through adaptive layer normalization via shift, scale, and gate
parameters. The final layer outputs a $27$-dimensional vector-field prediction
at each sequence position.

\paragraph{Training protocol.}
In the training stage, the base DiT model is trained with AdamW using learning
rate $3 \times 10^{-4}$, weight decay $0.1$, $\beta_1=0.9$, and
$\beta_2=0.999$. We train for $750{,}000$ optimization steps with global batch
size $512$, using linear warmup for the first $1000$ steps followed by a
reduce-on-plateau learning-rate schedule with minimum learning rate $10^{-5}$.
The global gradient norm is clipped to $1.0$.

In the reflow stage, recoupled endpoint pairs are generated from the previous
model by ODE sampling from $t=0$ to $t=1$ with tolerances
$\mathrm{atol}=10^{-5}$ and $\mathrm{rtol}=10^{-5}$. The reflow model is
initialized from the previous checkpoint and fine-tuned rather than trained
from scratch. We generate $100{,}000$ recoupled pairs in batches of $256$ and
train the reflow model for $40{,}000$ optimization steps with batch size $256$
and learning rate $4 \times 10^{-5}$. ReSFlow uses the spherical
parameterization induced by the square-root map, while R-LFM uses the
Euclidean reflow parameterization. Training was performed on $4$ NVIDIA A800
GPUs and took approximately $6$ days.

\subsection{Promoter Sequence Design}
\label{app:setup_promoter}

\paragraph{Dataset and preprocessing.}
We use the promoter sequence design dataset from the official Zenodo data
package\footnote{\url{https://doi.org/10.5281/zenodo.7943307}}, which includes
the GRCh38 reference genome, FANTOM transcription start site (TSS) annotations,
CAGE signal tracks, blacklist regions, and the pre-trained Sei model for
downstream evaluation. Each example is a $1024$ bp DNA sequence over the
alphabet $\{\mathrm{A},\mathrm{C},\mathrm{G},\mathrm{T}\}$, represented as a
one-hot tensor in $\mathbb{R}^{1024 \times 4}$; ambiguous bases are encoded as
the uniform vector $(1/4,1/4,1/4,1/4)$. The data split is chromosome-based:
after taking the first $100{,}000$ FANTOM TSS entries, chromosomes
$\mathrm{chr8}$ and $\mathrm{chr9}$ are used for testing, $\mathrm{chr10}$ for
validation, and all remaining chromosomes for training.

The conditioning signal is the FANTOM CAGE signal aligned to the same $1024$ bp
window as the DNA sequence. We use the plus-strand CAGE channel as a
per-position scalar condition, giving a conditioning tensor in
$\mathbb{R}^{1024 \times 1}$. During training, the TSS position is randomly
shifted by an offset sampled from $\{-10,\ldots,10\}$ before extracting the
sequence and signal window; no offset is used for validation or testing. The
flow prior $p_0$ is the categorical simplex prior, sampled independently at each
sequence position and normalized to match the $1024 \times 4$ sequence
representation, while the target distribution is the empirical distribution of
one-hot promoter sequences paired with their CAGE conditions.

\paragraph{Model architecture.}
The vector-field predictor for promoter design is a one-dimensional dilated
convolutional residual network. It takes the current categorical state
$x_t \in \mathbb{R}^{1024 \times 4}$ and the CAGE condition
$c \in \mathbb{R}^{1024 \times 1}$, concatenates them into a $5$-channel input,
and maps them to $256$ hidden channels with an initial kernel-size-$9$
convolution. The backbone contains $20$ residual convolutional blocks with
$256$ channels, kernel size $9$, GroupNorm with one group, Swish activations,
and additive time conditioning. The dilation pattern is
$(1,1,4,16,64)$ repeated four times. Time is embedded with Gaussian Fourier
features followed by a linear layer and SiLU activation, producing a
$256$-dimensional embedding that is added to each block after projection. The
CAGE signal is not separately encoded, but is injected only by input
concatenation. A final $1\times1$ convolutional head with GELU activation maps
$256 \to 256 \to 4$, yielding a four-channel vector-field prediction at each
sequence position. No dropout is used. All promoter experiments use this same
backbone. The differences between the base and reflow models come from the flow
geometry and reflow procedure rather than from the sequence encoder.

\paragraph{Training protocol.}
For the base promoter model, we train with Adam using learning rate
$5 \times 10^{-4}$, $\beta_1=0.9$, $\beta_2=0.999$, and no weight decay. The
batch size is $256$, and training runs for $200{,}000$ optimization steps with
gradient clipping at global norm $100$ and random seed $42$. We do not use
learning-rate warmup or exponential moving average. Validation is performed
every $5{,}000$ steps using the corresponding flow-matching objective on the
validation split, and checkpoints are saved periodically.

For ReSFlow, we perform a single reflow stage on top of the base model. The
reflow model is initialized from the base checkpoint and fine-tuned rather than
trained from scratch. We generate $100{,}000$ recoupled endpoint pairs by ODE
sampling from prior samples, integrating the learned vector field from $t=0$ to
$t=1$ with batch size $128$ and tolerances
$\mathrm{atol}=\mathrm{rtol}=10^{-5}$. During reflow pair generation, the CAGE
condition $c$ is kept fixed: for each conditioned prior sample $(x_0,c)$, we
integrate the conditional vector field to obtain the paired endpoint $x_1$.
The resulting recoupled dataset therefore consists of triples $(x_0,x_1,c)$.
The reflow stage is trained for $100{,}000$ optimization steps with batch size
$128$ and learning rate $10^{-4}$.

\subsection{Evaluation Metrics, NFE, and Sampling Protocols}
\label{app:metrics}

We describe the evaluation metrics and sampling protocols used in the main
experiments. Since our model represents categorical variables as points on the
probability simplex and performs flow matching after the square-root embedding,
we distinguish likelihood-based evaluation, sample-quality evaluation, and
task-specific downstream evaluation.

\paragraph{Negative log-likelihood.}
For the Swiss-roll experiment, we report negative log-likelihood (NLL) using
the continuous-time change-of-variables formula associated with the learned
ODE. This evaluation treats the data as samples from a continuous density on
the simplex, or equivalently on the positive orthant of the sphere after the
square-root map. Therefore, unlike discrete cross-entropy, the resulting NLL is
a differential likelihood quantity and is not constrained to be non-negative.
Accurate NLL estimation requires numerical integration of the learned dynamics,
and we therefore evaluate NLL only with adaptive ODE integration rather than
with one-step sampling.

\paragraph{Fréchet Inception Distance.}
For binarized MNIST, we report Fréchet Inception Distance (FID) as a
sample-quality metric. Generated binary images are compared with the reference
data distribution using the standard FID protocol; lower values indicate that
the generated samples better match the real data distribution. We report FID
both under adaptive ODE sampling and under one-step sampling. This allows us to
separate the quality of the learned continuous-time flow from the quality of
the compressed one-step sampler.

\paragraph{Bits per character.}
For Text8, we report bits per character (BPC). Since ReSFlow and the compared
flow-matching baselines are not autoregressive language models, we do not
evaluate BPC by multiplying token-wise autoregressive conditionals. Instead, we
use a variational integral objective that measures how well the learned vector
field predicts the terminal categorical state from an intermediate point along
the probability path.

Let $\mu_1 \sim q$ be a data sequence and $\mu_0 \sim p_0$ be a prior sample.
We map them to the sphere as $X_0=\pi(\mu_0)$ and $X_1=\pi(\mu_1)$, and define
the intermediate point by spherical geodesic interpolation
\[
    X_t = \mathrm{slerp}(X_0,X_1;t), \qquad
    \mu_t = \pi^{-1}(X_t).
\]
Given $X_t$, the model predicts the endpoint by extrapolating the learned
velocity field from time $t$ to time $1$ on the sphere:
\[
    \widehat{X}_1(X_t,t)
    =
    \mathrm{Exp}_{X_t}\!\left((1-t)\,v_\theta(X_t,t)\right),
    \qquad
    \widehat{\mu}_1(\mu_t,t)
    =
    \pi^{-1}\!\left(\widehat{X}_1(X_t,t)\right).
\]
For a sequence of length $L$, the variational negative log-likelihood is
computed as
\[
    \mathcal{L}_{\mathrm{var}}
    =
    -
    \mathbb{E}_{\mu_0\sim p_0,\,\mu_1\sim q}
    \left[
    \int_0^1
    \frac{1}{1-t}
    \sum_{\ell=1}^{L}
    \log
    \left\langle
        \widehat{\mu}_{1,\ell}(\mu_t,t),
        \mu_{1,\ell}
    \right\rangle
    dt
    \right],
\]
where $\langle\cdot,\cdot\rangle$ denotes the standard Euclidean inner product
over the categorical alphabet at each sequence position. Since $\mu_{1,\ell}$
is one-hot for Text8, the inner product selects the probability assigned to the
ground-truth character at position $\ell$.

The factor $1/(1-t)$ makes direct numerical integration near $t=1$ unstable.
Following the standard variational evaluation used for continuous-time
categorical flows, we apply the change of variables
$s=-\log(1-t)$, or equivalently $t=1-e^{-s}$. This gives
\[
    \mathcal{L}_{\mathrm{var}}
    =
    -
    \mathbb{E}_{\mu_0\sim p_0,\,\mu_1\sim q}
    \left[
    \int_0^\infty
    \sum_{\ell=1}^{L}
    \log
    \left\langle
        \widehat{\mu}_{1,\ell}(\mu_t,t),
        \mu_{1,\ell}
    \right\rangle
    ds
    \right],
    \qquad t=1-e^{-s}.
\]
In practice, the integrand becomes negligible for large $s$ because $t$ is then
very close to $1$. We truncate the integral at a finite upper limit and
evaluate it numerically with an adaptive solver. The reported BPC is obtained
by normalizing the variational negative log-likelihood by the sequence length
and converting from nats to bits:
\[
    \mathrm{BPC}
    =
    \frac{\mathcal{L}_{\mathrm{var}}}{L\log 2}.
\]
This metric is likelihood-based and therefore is reported only for ODE-based
evaluation, not for one-step sample generation.

\paragraph{Signal proxy mean squared error.}
For conditional promoter sequence design, we use signal proxy mean squared
error (SP-MSE). Each generated promoter sequence is evaluated by a pre-trained
sequence-to-function predictor, and the predicted promoter activity is compared
with that of real human promoter sequences on the test chromosomes. Denoting by
$r_i$ the predicted activity profile of the real promoter and by
$\widehat{r}_i$ the predicted activity profile of the generated promoter under
the same conditioning signal, SP-MSE is computed as
\[
    \mathrm{SP\mbox{-}MSE}
    =
    \frac{1}{N}
    \sum_{i=1}^{N}
    \left\|
        \widehat{r}_i - r_i
    \right\|_2^2 .
\]
Lower SP-MSE indicates that the generated DNA sequences better preserve the
desired promoter activity implied by the conditioning signal. Unlike FID or
NLL, this is a downstream functional metric: it evaluates whether generated
sequences are consistent with the biological signal rather than only measuring
distributional similarity in sequence space.

\paragraph{Number of function evaluations.}
We use NFE to denote the number of evaluations of the learned vector field
during sampling. Adaptive Dormand--Prince sampling uses a variable number of
function evaluations determined by the ODE solver tolerance. In contrast,
one-step generation uses NFE$=1$. For ReSFlow, one-step sampling evaluates the
velocity field once at $t=0$ and maps the result to the endpoint using the
spherical exponential map:
\[
    \widehat{X}_1
    =
    \mathrm{Exp}_{X_0}\!\left(v_\theta(X_0,0)\right),
    \qquad
    \widehat{\mu}_1 = \pi^{-1}(\widehat{X}_1).
\]
The final discrete sample is obtained from $\widehat{\mu}_1$ by sampling from,
or taking the maximum-probability category of, the predicted categorical
distribution.

\paragraph{Wall-clock sampling time.}
We report wall-clock sampling time for binarized MNIST and promoter sequence
design to quantify practical inference efficiency. Timing is measured for
generating $1{,}000$ samples and includes the sampling procedure only; it does
not include training, checkpoint loading, or metric computation. For promoter
design, the reported time corresponds to generating $1{,}000$ sequences of
length $1024$. We do not report Text8 sampling time in the main table because
the reported BPC is computed from the variational likelihood integral rather
than from generated samples.

\clearpage
\section{Additional Results}
\label{app::addi}
\subsection{FID versus NFE on Binarized MNIST}
\begin{figure}[H]
\centering
\includegraphics[width=0.75\linewidth]{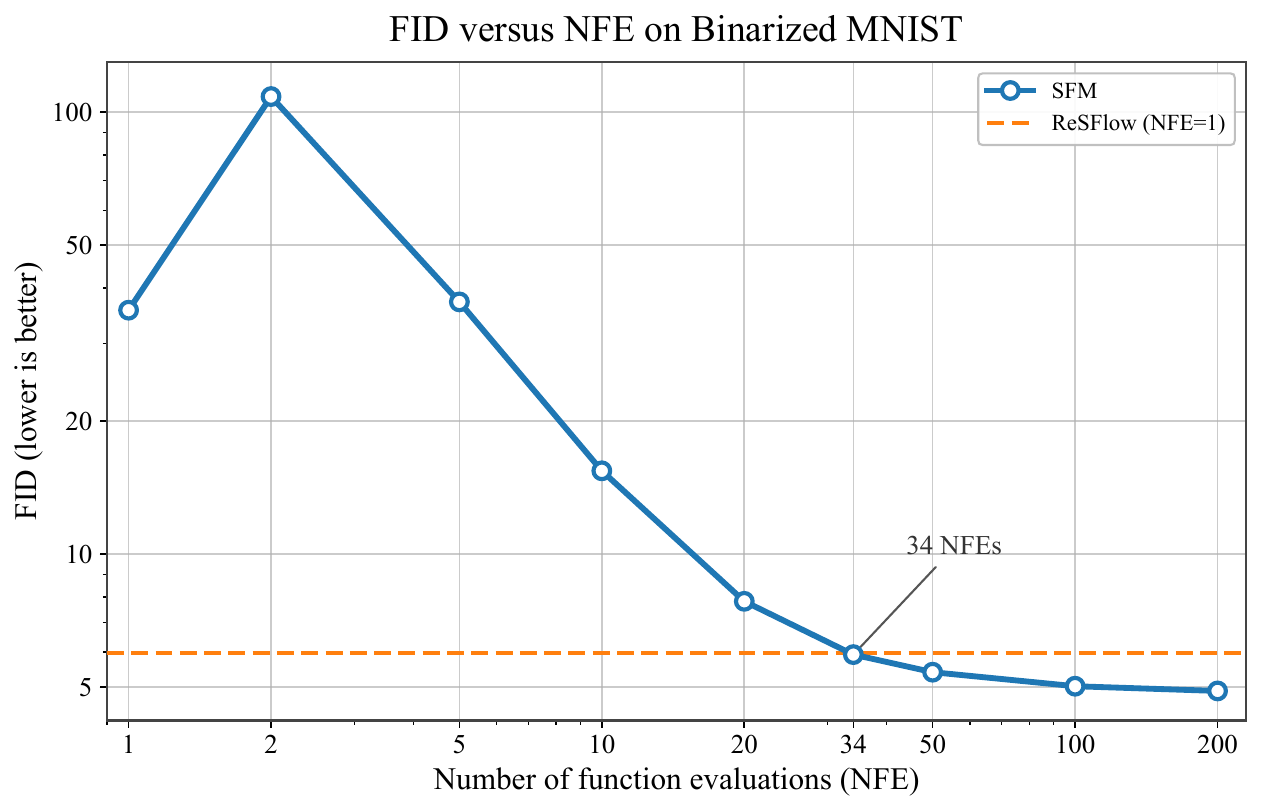}
\caption{Binarized MNIST FID versus NFE using the matched spherical exponential-map Euler sampler, three seeds, and 1,000 generated samples per seed. The solid curve shows the SFM mean; the dashed line marks the one-step ReSFlow mean ($5.973$). SFM first has a lower mean at NFE$=34$.}
\label{fig:bmnist_nfe}
\end{figure}

\subsection{Fisher--Rao Transport Cost Analysis of Reflow Couplings}
\label{subsec:fr_transport_cost_reflow}

\begin{figure}[H]
    \centering
    \includegraphics[width=0.75\linewidth]{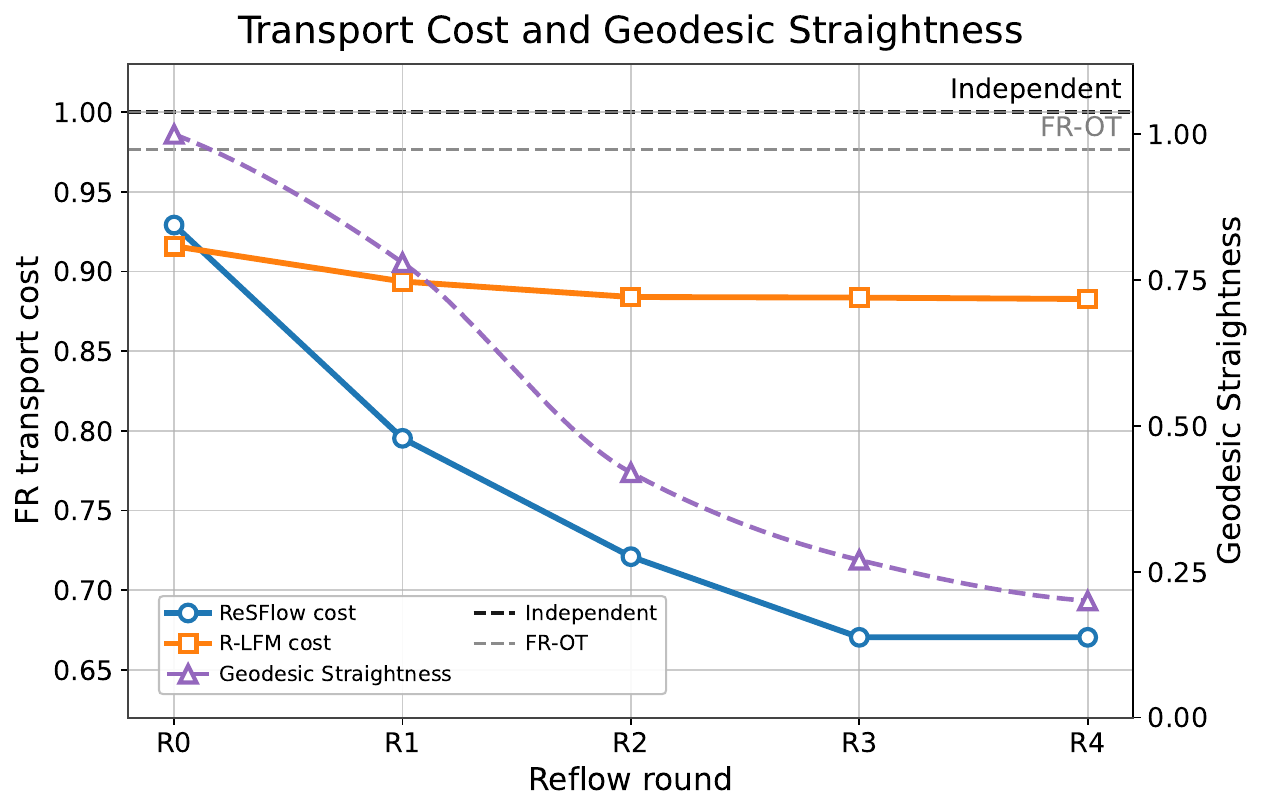}
    \caption{
    Fisher--Rao transport cost and Geodesic Straightness on BMNIST.
    }
    \label{fig:fr_transport_cost_reflow}
\end{figure}

To further examine why spherical reflow improves one-step generation, we analyze the geometry of the couplings produced during reflow on BMNIST. Since each pixel is represented as a categorical distribution on the probability simplex, Euclidean distance is not the appropriate notion of discrepancy. We therefore measure whether reflow produces shorter and more geodesically consistent couplings under the intrinsic Fisher--Rao geometry. This analysis complements the previous sampling results by directly evaluating whether rectification simplifies the transport problem on the simplex.

Following the Fisher--Rao distance defined earlier, for two categorical distributions $\mu,\nu\in\Delta^{C-1}$, we denote this distance by
\begin{equation}
d_{\text{cat}}(\mu,\nu)
=
2\arccos\left(
\sum_{c=1}^{C}\sqrt{\mu_c\nu_c}
\right).
\end{equation}
For a full image with $D$ pixels, the corresponding Fisher--Rao transport cost between a pair $(z_0,z_1)$ is defined as
\begin{equation}
\mathcal{C}_{\mathrm{FR}}(z_0,z_1)
=
\sum_{j=1}^{D}
d_{\text{cat}}(z_{0,j},z_{1,j}),
\end{equation}
where $z_{0,j}$ and $z_{1,j}$ denote the categorical distributions at pixel $j$. We compare the reflow couplings against two baselines shown in Fig.~\ref{fig:fr_transport_cost_reflow}: an independent random coupling, where prior and data samples are paired independently, and a mini-batch Fisher--Rao optimal transport coupling (FR-OT), where pairings are optimized within each mini-batch using the same Fisher--Rao cost. We also report the costs of the couplings produced by ReSFlow and by the Euclidean reflow variant R-LFM.

As shown in Fig.~\ref{fig:fr_transport_cost_reflow}, the independent coupling has the largest transport cost, indicating that arbitrary pairings between prior samples and data samples are geometrically inefficient. FR-OT reduces this cost by explicitly selecting lower-cost assignments within each mini-batch, but the improvement is limited, suggesting that mini-batch optimal transport only provides a locally optimized coupling and cannot substantially simplify the global transport structure. Notably, across reflow rounds, ReSFlow substantially reduces the Fisher--Rao transport cost, whereas R-LFM only achieves a mild reduction and quickly plateaus. This indicates that the improvement of ReSFlow comes from rectifying the coupling in the square-root hyperspherical representation induced by the Fisher--Rao metric, where categorical distributions are transported along intrinsic spherical directions. However, Euclidean reflow updates the coupling in a geometry that is not aligned with the Fisher--Rao structure, leading to much weaker reduction in intrinsic transport cost. These results support our central claim that hyperspherical flow modeling and reflow rectification are complementary for simplifying transport on the probability simplex, empirically corroborating the spherical transport-cost contraction in \Cref{thm:cost_contraction}.

\subsection{Riemannian Geodesic Straightness of Learned Trajectories}
\label{subsec:riemannian_geodesic_straightness}

We further evaluate a complementary trajectory-level property:
\emph{whether the actual ODE trajectories induced by the learned vector field
become closer to geodesics after reflow}. To this end, we design a Riemannian
analogue of the straightness metric: 

\paragraph{Metric design.}
In Euclidean Rectified Flow, trajectory straightness is measured by comparing the learned velocity along an ODE trajectory with the displacement between the trajectory’s initial and terminal points.
On a curved statistical manifold,
however, the Euclidean displacement vector is not intrinsic. Since ReSFlow
operates on the hypersphere $\mathbb{S}^{K-1}_+$ through the square-root map
$\pi$, the natural analogue of a straight line is a spherical geodesic, and the
natural analogue of a displacement vector is the Riemannian logarithmic map.

For the $k$-th reflow checkpoint, let  $v_{\hat{\theta}}^{(k)}$, denote the learned
velocity field. Starting from an initial point
$Z^{(k)}_0 \sim \pi(p_0)$, we generate the actual model trajectory by
integrating
\begin{equation}
    \frac{d Z^{(k)}_t}{dt}
    =
    v_{\hat{\theta}}^{(k)}(Z^{(k)}_t,t),
    \qquad t\in[0,1].
\end{equation}
We denote the terminal point of this integrated trajectory by
$Z^{(k)}_1$. Importantly, $Z^{(k)}_1$ is the endpoint produced by the learned
ODE itself, rather than an externally specified data sample. Therefore, the
metric evaluates the intrinsic straightness of the model trajectory.

At an intermediate point $Z^{(k)}_t$, the tangent vector of the
constant-speed geodesic that reaches the same endpoint $Z^{(k)}_1$ in the
remaining time $1-t$ is
\begin{equation}
    u_{\mathrm{geo}}^{(k)}(Z^{(k)}_t,t)
    =
    \frac{
    \Log_{Z^{(k)}_t}
    \left(
    Z^{(k)}_1
    \right)
    }{1-t}.
\end{equation}
We define the Riemannian geodesic straightness of the $k$-th checkpoint as
\begin{equation}
\label{eq:riem_geodesic_straightness}
    S_{\mathrm{Riem}}
    \left(
    Z^{(k)}
    \right)
    =
    \int_0^1
    \mathbb{E}_{Z^{(k)}_0\sim\pi(p_0)}
    \left[
    \left\|
        v_{\hat{\theta}}^{(k)}(Z^{(k)}_t,t)
        -
        \frac{
        \Log_{Z^{(k)}_t}
        (
        Z^{(k)}_1
        )
        }{1-t}
    \right\|^2_{g}
    \right]
    dt .
\end{equation}
Equivalently, this metric replaces the Euclidean displacement used in the
original straightness measure by the instantaneous tangent vector required to
move from the current point to the trajectory endpoint along a geodesic. If the
learned trajectory itself is a constant-speed geodesic, then
$S_{\mathrm{Riem}}(Z^{(k)})=0$. Hence, a smaller value indicates that the
learned ODE path is more geodesic and more suitable for one-step generation via
the exponential map.

\paragraph{Results on binarized MNIST.}
We evaluate the Riemannian geodesic straightness on binarized MNIST across
successive reflow rounds, as reported in
Figure~\ref{fig:fr_transport_cost_reflow}. The straightness metric decreases
monotonically as the number of reflow rounds increases. This indicates that
reflow progressively straightens the learned trajectories on the statistical
manifold. The decreasing trend directly supports the geometric interpretation
of ReSFlow in the main text and empirically validates the straightening-rate
guarantee in \Cref{thm:straightening_rate}: reflow makes the learned ODE
trajectories increasingly aligned with Fisher--Rao geodesics, thereby
explaining why a single exponential-map update can achieve strong one-step
generation performance.

\subsection{Generated Samples}
\label{app:generated-samples}

\begin{figure*}[t]
  \centering
  \noindent\makebox[\textwidth][c]{%
    \includegraphics[width=0.95\textwidth]{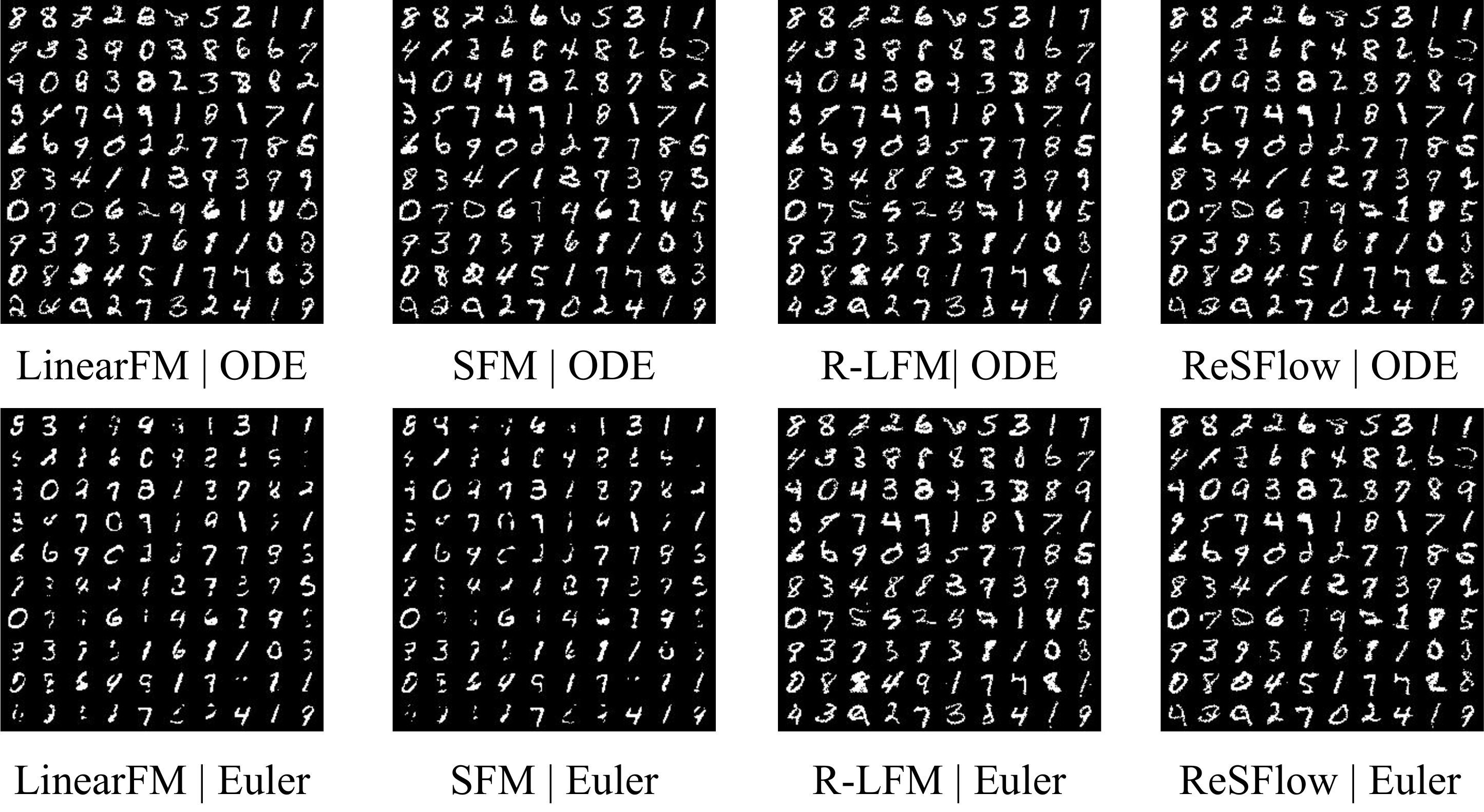}
  }
  \caption{
  \textbf{Generated samples on binarized MNIST.}
  We compare LinearFM, SFM, R-LFM, and ReSFlow under ODE and one-step Euler
  sampling. Non-rectified methods degrade under one-step sampling, whereas
  rectified methods preserve digit structure more effectively. ReSFlow gives
  high-quality one-step samples by rectifying trajectories under the
  Fisher--Rao geometry.
  }
  \label{fig:app_bmnist_samples}
\end{figure*}

\begin{table*}[htbp]
\centering
\caption{
\textbf{Generated Text8 samples.}
We show representative character-level samples generated by each model under
ODE and one-step Euler sampling. Underscores denote spaces in the decoded text.
}
\label{tab:app_text8_samples}
\scriptsize
\begin{tabular}{llp{0.76\textwidth}}
\toprule
Model & Sampler & Generated text \\
\midrule

LinearFM & ODE &
\path|calls_vincenti_or_their_burnt_mochery_a_distinct_writing_artiast_is_also_known_as_the_yodioi_cascoliage_it_is_a_crutter_with_bianaple_it_is_known_as_a_while_green_hole_and_even_when_using_plhuhu_snow_and_wouldistay_is_one_of_several_ancient_nuts_of_early_w| \\

LinearFM & Euler &
\path|_______________________________________________________________e_______________________________________________________________________________________________e___e_______________________e____________________________________________________________________| \\

\midrule

SFM & ODE &
\path|tters_improvements_a_primary_work_to_most_word_genetic_cinematics_who_also_chose_a_word_for_genetic_generation_concurrence_with_a_pt_typical_sporting_the_world_is_often_the_most_for_compatisans_works_such_as_sorcery_and_the_dark_chasen_movements_the_player| \\

SFM & Euler &
\path|__izg_i_ee_e_e__fe______________t_______g_t__e__a_ote_____e____e_________e___d__a__i______t__q__t__tebe_s______e____ee___s__d_i_oe_i__s___ee__ene___o_e__tei__eoe___e_n_____e______a_ea__a__yer_ie__ei_teein__s_n__________e____o___iu___n__eo___f__ag______ho__| \\

\midrule

R-LFM & ODE &
\path|_ners_vincento_wrata_orebry_n_mother_held_at_introducing_artiary_tale_two_beek_in_the_change_casckling_and_is_spyr_operating_fiama_of_rownspart_s_symbols_large_numbers_and_even_opecause_of_his_personalny_would_state_there_in_several_day_baptists_of_eyric_w| \\

R-LFM & Euler &
\path|ieto__uaibe_and_the_early_ver_ion_an__a_aeementary_tounterlaei_eoond_the_forst_medonies_ere_ted_from_greatly_attended_from_the_naen_so_condonati_n_tiliter__head_tittre_ont_weec_sters_tore_tad_speecres_and_which_table_to_remume_expending_that_the_leies_anni| \\

\midrule

ReSFlow & ODE &
\path|another_role_ranged_in_from_one_eight_two_one_le_conseliation_of_louis_nweed_d_law_in_trlits_quits_teberry_in_children_muslid_born_herman_lempenn_one_eight_five_four_whischeuse_leaders_of_several_differences_of_privileges_sion_new_mans_to_the_war_knew_his_| \\

ReSFlow & Euler &
\path|into_swy_no_and_they_are_also_itn_ha_is_also_ttaught_una__fee_be_onf_the_popular_oo_is_lireaand_for_exampard_altended_fr_a_aerenaen_so_land_external_links_the_d_rib_ol_blt_deer_opern_type_tad_seelanes_and_to_cannabse__k_vemuns_i__endite_fere_the_level_ard_| \\

\bottomrule
\end{tabular}
\end{table*}

\begin{table*}[htbp]
\centering
\caption{
\textbf{Generated promoter sequence samples.}
Each generated promoter has length $1024$. For readability, we show only the
first $120$ bases of one selected sample for each model and sampler. Spaces are
inserted every $40$ bases only for line breaking and are not part of the DNA
sequence. Quantitative metrics are computed using the full-length sequences.
}
\label{tab:app_promoter_samples}
\scriptsize
\begin{tabular}{llp{0.72\textwidth}}
\toprule
Model & Sampler & First 120 bases of generated promoter sequence \\
\midrule

LinearFM & ODE &
\texttt{AAGAGCTCGGCCGCGGGCTCGGGTGCCCTCCCCCCGGCCC}
\texttt{ CAAAGCCCCGGGGATTTCCAAGGCCGGCAAGGCCGGGGGG}
\texttt{ CAGGGGGGAAGAGGGGCGGCACTGGGAGGGGGGATGCGAA} \\

\addlinespace[0.6em]

LinearFM & Euler &
\texttt{CTCGCCATCGGGCAGGGGGCTATTTTGCATTCTAACGCCA}
\texttt{ GGCAGTAGTCGTTTGGTCGCATGCAAACTCAGATGCGGGG}
\texttt{ GATTGGGTGGCATTCGGTCTCGGCGTTAGACCGAAACCAG} \\

\midrule

SFM & ODE &
\texttt{AAGAGCTAAGCCGCGGGCTCAGGTGCCCTCCCCCCAGCCC}
\texttt{ CAAAGCCCCGGGGAGCTCCAACGCCACCACCCCCAGGCGG}
\texttt{ CAGACCGGAACACGGGCGCCACTGCCAGGAAGGATGCGAA} \\

\addlinespace[0.6em]

SFM & Euler &
\texttt{AGGGGCCAAGCCCCGGCGGCGCGTGCGGGGCCGTCAGCTT}
\texttt{ CAGCCCCCCAGCGGATCCCGTAGACGCCAGCCCCAGGCGG}
\texttt{ CACTCCGGGACGCGGGCGCCGTCCCGGAGAAGGGCGCGGA} \\

\midrule

R-LFM & ODE &
\texttt{AAGAACTCCGCCGCGGGCTCGGGTGCCCTCCCCCCGGCCC}
\texttt{ CAAAGCCGCGGGGAACGCCACCGCCACCACAGCCGGCCGG}
\texttt{ CAGAAGGGAACAGGGGCGGCCCTAGCAGGCGGGAGGCGAA} \\

\addlinespace[0.6em]

R-LFM & Euler &
\texttt{AAGAACTCGGCCGCGGGCTCAGGTGCACTACCCCCGGCCC}
\texttt{ CAAAGCCGCGGGAAACGCCACCGCCAGCAAAGCCGGGGGG}
\texttt{ CAGAAGGGAACAGGGGCGGCCCTAGGAGGGGGGAGGCGAA} \\

\midrule

ReSFlow & ODE &
\texttt{AAGAGCTAAGCCGCGGGCTCAGGTGCCCTCCCCCCAGCCC}
\texttt{ CACAGCCCCGGGGAGCTCCAACGCCACCACCCCCAGGCGG}
\texttt{ CAGCCCGGAACACGGGCGCCCCTGCCAGGAAGGATGCGAA} \\

\addlinespace[0.6em]

ReSFlow & Euler &
\texttt{AAGACCTAAGCCGCGGGCTCAGGTGCCCTCCCCCCAGCCC}
\texttt{ CACAGCCCCGGGGAGCTCCAACGCCACCACCCCCAGGCGG}
\texttt{ CAGCCCGGAACACGGGCGCCCCTGCCAGGAAGGATGCGAA} \\

\bottomrule
\end{tabular}
\end{table*}

We provide additional qualitative samples for binarized MNIST, Text8, and
conditional promoter sequence design in Fig.~\ref{fig:app_bmnist_samples},
Table~\ref{tab:app_text8_samples}, and Table~\ref{tab:app_promoter_samples},
respectively. These examples complement the quantitative results by comparing
ODE sampling with one-step Euler sampling.

For binarized MNIST, Fig.~\ref{fig:app_bmnist_samples} shows samples generated
by LinearFM, SFM, R-LFM, and ReSFlow. Non-rectified methods produce reasonable
ODE samples but degrade noticeably under one-step sampling, while rectified
methods better preserve digit structure. ReSFlow gives clear one-step samples,
supporting the benefit of performing rectification under the Fisher--Rao
geometry.

For Text8, Table~\ref{tab:app_text8_samples} reports representative generated
character sequences from each model and sampler. The one-step samples highlight
whether the learned flow is sufficiently straight for fast generation. ReSFlow
produces more structured character-level fragments than the non-rectified
baselines, consistent with the BPC results in Table~\ref{tab:text8_results}.

For promoter sequence design, Table~\ref{tab:app_promoter_samples} shows the
first $120$ bases of generated $1024$-bp promoter sequences. Only prefixes are
shown for readability, while all quantitative metrics are computed using the
full-length sequences. ReSFlow maintains stable nucleotide patterns under
one-step sampling, consistent with its SP-MSE performance in
Table~\ref{tab:promoter_results}.
\section{Limitations and Broader Impact}
While our manifold-based formulation successfully captures the intrinsic geometry of discrete data, it relies on model assumptions that may be violated in complex real-world scenarios. For example, the Fisher--Rao geometry inherently assumes semantic equidistance—treating all categorical states symmetrically without prior semantic bias. In reality, discrete tokens (e.g., words or subwords) possess rich, hierarchical semantic proximity (e.g., ``cat'' is conceptually closer to ``dog'' than to ``apple''). While our method prevents the geometric distortion seen in continuous embedding methods (like ReDi), it does not currently incorporate pre-existing semantic priors. The assumption that pure statistical distances are sufficient may limit performance on tasks where capturing pre-trained semantic correlations is crucial.

ReSFlow is a foundational methodological contribution to generative modeling on the probability simplex, and is not tied to any particular deployment. On the positive side, it offers a geometry-aware and sampling-efficient framework that can benefit discrete generative tasks in language, biology (e.g., DNA/protein sequences), and other scientific domains where categorical data are central, while its straightening property reduces inference cost and the associated energy footprint. On the negative side, ReSFlow shares the generic risks of all improved generative models: higher-quality and faster discrete generation could in principle be misused for disinformation or other malicious content. 
\newpage
\input{checklist.tex}

\end{document}

%% file: theories.tex
\section{Spherical Rectified-Flow Proofs}
\label{app:spherical-proofs}

For a random variable \(X\), \(\Law(X)\) denotes its probability distribution.
This appendix follows the structure of the rectified-flow analysis
\cite{liu2023flow}, but replaces the Euclidean linear interpolation by
spherical geodesic interpolation and makes explicit the geometric steps needed
on \(M=\mathbb{S}^{K-1}\).  The first subsection collects the sphere geometry
and conditional-expectation identities used below.  The second subsection proves
the spherical marginal-preservation property.  The third subsection shows that
spherical rectification reduces convex nondecreasing functions of the geodesic
transport distance.  The final subsection states the straightening effect and
the reflow \(O(1/N_{\text{reflow}})\) rate in the same spirit as the original rectified-flow
paper, with the corresponding spherical notions of straightness and energy.

\subsection{Geometric and Probabilistic Preliminaries}

\paragraph{Sphere Geometry.}

Let \(M=\mathbb{S}^{K-1}=\{x\in\mathbb{R}^{K}:\|x\|_2=1\}\).  We use the
Riemannian metric inherited from \(\mathbb{R}^{K}\).  Thus the tangent space and
inner product are
\[
    T_xM=\{u\in\mathbb{R}^{K}:u^\top x=0\}, \qquad
    \langle u,v\rangle_x=u^\top v, \qquad \|u\|_x=\sqrt{u^\top u}.
\]
The geodesic distance is
\[
    d_{\mathbb{S}}(x,y)=\arccos(x^\top y)\in[0,\pi].
\]
The cut locus of \(x\) is the antipodal point \(-x\).  When \(y\neq -x\), the
logarithmic map is single-valued.  Writing
\(\theta=d_{\mathbb{S}}(x,y)\), for \(0<\theta<\pi\),
\begin{equation}
    \log_x(y)=\frac{\theta}{\sin\theta}(y-\cos\theta\,x)\in T_xM,
    \qquad \|\log_x(y)\|_x=\theta .
    \label{eq:app-log-formula}
\end{equation}
For \(y=x\), we set \(\log_x(x)=0\).  The exponential map is
\[
    \exp_x(u)=
    \cos(\|u\|_x)x+\sin(\|u\|_x)\frac{u}{\|u\|_x},
    \qquad u\in T_xM,\ u\neq 0,
\]
with \(\exp_x(0)=x\).

\begin{lemma}[Constant-speed spherical geodesic]
\label{lem:app-geodesic}
Let \(x,y\in M\) with \(y\neq -x\).  The curve
\[
    \gamma_t=\exp_x(t\log_x(y)),\qquad t\in[0,1],
\]
is the minimizing geodesic from \(x\) to \(y\).  It has the slerp form
\begin{equation}
    \gamma_t
    =
    \frac{\sin((1-t)\theta)}{\sin\theta}x
    +
    \frac{\sin(t\theta)}{\sin\theta}y,
    \qquad \theta=d_{\mathbb{S}}(x,y),
    \label{eq:app-slerp}
\end{equation}
and it has constant speed
\[
    \|\dot\gamma_t\|_{\gamma_t}=\theta=d_{\mathbb{S}}(x,y).
\]
\end{lemma}

\begin{proof}
For \(0<\theta<\pi\), define
\[
    e=\frac{y-\cos\theta\,x}{\sin\theta}.
\]
Then \(e^\top x=0\), \(\|e\|=1\), and \(y=\cos\theta\,x+\sin\theta\,e\).  By
\Cref{eq:app-log-formula}, \(\log_x(y)=\theta e\), so
\[
    \gamma_t=\exp_x(t\theta e)=\cos(t\theta)x+\sin(t\theta)e.
\]
Substituting the definition of \(e\) gives \Cref{eq:app-slerp}.  Differentiating
gives \(\dot\gamma_t=\theta[-\sin(t\theta)x+\cos(t\theta)e]\), whose norm is
\(\theta\).  The curve stays in the two-dimensional plane spanned by \(x\) and
\(y\) and traces the shorter great-circle arc, hence it is the minimizing
geodesic.  The case \(x=y\) is the constant curve.
\end{proof}

\begin{lemma}[Path length and energy bounds]
\label{lem:app-length-energy}
If \(z_t\) is any absolutely continuous curve on \(M\), then
\begin{equation}
    d_{\mathbb{S}}(z_0,z_1)
    \le
    \int_0^1 \|\dot z_t\|_{z_t}\,dt
    \label{eq:app-length}
\end{equation}
and
\begin{equation}
    d_{\mathbb{S}}(z_0,z_1)^2
    \le
    \int_0^1 \|\dot z_t\|_{z_t}^2\,dt .
    \label{eq:app-energy}
\end{equation}
\end{lemma}

\begin{proof}
The geodesic distance is the minimum length over all curves connecting the two
endpoints, so it is no larger than the length of \(z_t\), proving
\Cref{eq:app-length}.  Squaring \Cref{eq:app-length} and using
Cauchy--Schwarz,
\[
    \left(\int_0^1 \|\dot z_t\|_{z_t}\,dt\right)^2
    \le
    \left(\int_0^1 1^2dt\right)
    \left(\int_0^1 \|\dot z_t\|_{z_t}^2dt\right),
\]
which gives \Cref{eq:app-energy}.
\end{proof}

\paragraph{Probabilistic Identities.}

All conditional expectations below are taken after embedding tangent vectors in
\(\mathbb{R}^{K}\).  This creates no ambiguity because each tangent space is a
linear subspace of \(\mathbb{R}^{K}\).

\begin{lemma}[Conditional tangent vectors]
\label{lem:app-conditional-tangent}
Let \(X\) be an \(M\)-valued random variable and let \(U\) be a random vector
with \(U\in T_XM\) almost surely and \(\mathbb{E}[\|U\|^2]<\infty\).  Then
\[
    m(X):=\mathbb{E}[U\mid X]\in T_XM \qquad \text{a.s.}
\]
Moreover,
\[
    \mathbb{E}[\|U\|^2]
    =
    \mathbb{E}[\|m(X)\|^2]
    +
    \mathbb{E}[\|U-m(X)\|^2] .
\]
\end{lemma}

\begin{proof}
Since \(U\in T_XM\), we have \(X^\top U=0\) a.s.  Conditioning on \(X\),
\[
    X^\top m(X)=\mathbb{E}[X^\top U\mid X]=0,
\]
so \(m(X)\in T_XM\).  For the decomposition, expand the square:
\[
    \|U\|^2
    =
    \|m(X)\|^2+\|U-m(X)\|^2
    +2\langle m(X),U-m(X)\rangle .
\]
The expectation of the cross term is zero because
\(\mathbb{E}[U-m(X)\mid X]=0\).
\end{proof}

\begin{lemma}[Convexity of radial costs]
\label{lem:app-radial-convex}
If \(\phi:[0,\infty)\to\mathbb{R}\) is convex and nondecreasing, then
\[
    F(w)=\phi(\|w\|)
\]
is convex on every Euclidean vector space.
\end{lemma}

\begin{proof}
For \(a,b\) in a Euclidean vector space and \(\lambda\in[0,1]\), the triangle
inequality gives
\[
    \|\lambda a+(1-\lambda)b\|
    \le
    \lambda\|a\|+(1-\lambda)\|b\|.
\]
Because \(\phi\) is nondecreasing and convex,
\[
    \phi(\|\lambda a+(1-\lambda)b\|)
    \le
    \phi(\lambda\|a\|+(1-\lambda)\|b\|)
    \le
    \lambda\phi(\|a\|)+(1-\lambda)\phi(\|b\|).
\]
\end{proof}

\subsection{Marginal Preservation}

\begin{definition}[Expected spherical velocity]
\label{def:app-expected-spherical-velocity}
For a path-wise continuously differentiable random process
\(X=\{X_t:t\in[0,1]\}\) on \(M\), its expected spherical velocity \(v^X\) is
defined as
\[
    v^X(x,t)=\mathbb{E}[\dot X_t\mid X_t=x]\in T_xM,
    \qquad x\in \operatorname{supp}(X_t).
\]
For \(x\notin\operatorname{supp}(X_t)\), we set \(v^X(x,t)\) arbitrarily as a
measurable tangent vector.  The velocity is tangent by
\Cref{lem:app-conditional-tangent}.
\end{definition}

\begin{definition}[Spherical rectifiability]
\label{def:app-spherical-rectifiable}
We call \(X=\{X_t:t\in[0,1]\}\) \emph{rectifiable} if the velocity field
\(v^X:M\times[0,1]\to TM\) is continuous in \((x,t)\) and locally Lipschitz
in \(x\), uniformly on compact time intervals.  Under this assumption, the
characteristic ODE
\begin{equation}
    \dot Z_t = v^X(Z_t,t),\qquad Z_0=X_0,
    \label{eq:app-spherical-ode}
\end{equation}
admits, by the Cauchy--Lipschitz theorem on the compact manifold \(M\)
\cite[Thm.~9.12]{lee2013smooth}, a unique global flow
\(\Phi:[0,1]\times M\to M\),\, \((t,x)\mapsto \Phi_t(x)\), which is a
\(C^1\) diffeomorphism for each \(t\).  We call
\(Z_t:=\Phi_t(X_0)\) the \emph{spherical rectified flow} induced by \(X\).
\end{definition}

\begin{theorem}[Spherical marginal preservation]
\label{thm:app-marginal-preservation}
Assume \(X=\{X_t:t\in[0,1]\}\) is rectifiable and \(Z\) is its spherical
rectified flow.  Then
\[
    \Law(Z_t)=\Law(X_t),\qquad \forall t\in[0,1].
\]
\end{theorem}

\begin{proof}
Let \(\mu_t=\Law(X_t)\) and \(\nu_t=\Law(Z_t)\).  We show
\(\mu_t=\nu_t\) by verifying that both satisfy the same continuity equation and applying a duality argument against arbitrary test functions \(h\in C^\infty(M)\).

\smallskip
\noindent\emph{Step 1 (\(\nu_t\) is the pushforward).}
Since \(Z_t=\Phi_t(X_0)\) and \(X_0\sim\mu_0\), the definition of the pushforward measure gives
\[
    \nu_t=(\Phi_t)_\#\mu_0,\qquad t\in[0,1].
\]

\smallskip
\noindent\emph{Step 2 (Weak formulation for \(\mu_t\)).}
For any \(h\in C^\infty(M)\), path-wise differentiability of \(X\) and the
tower property of conditional expectation yield
\begin{equation}
    \frac{d}{dt}\!\int_M h\,d\mu_t
    =\frac{d}{dt}\mathbb{E}[h(X_t)]
    =\mathbb{E}\!\left[\langle\nabla_{\mathbb{S}}h(X_t),\dot X_t\rangle_{X_t}\right]
    =\int_M \langle \nabla_{\mathbb{S}}h(x),v^X(x,t)\rangle_x\,d\mu_t(x).
    \label{eq:app-weak-mu}
\end{equation}

\smallskip
\noindent\emph{Step 3 (Weak formulation for \(\nu_t=(\Phi_t)_\#\mu_0\)).}
Apply the chain rule to \(h(\Phi_t(x))\) and integrate against \(\mu_0\):
\begin{align}
    \frac{d}{dt}\!\int_M h\,d\nu_t
    &= \frac{d}{dt}\!\int_M h(\Phi_t(x))\,d\mu_0(x)
    = \int_M \langle\nabla_{\mathbb{S}}h(\Phi_t(x)),v^X(\Phi_t(x),t)\rangle_{\Phi_t(x)}\,d\mu_0(x)\nonumber\\
    &= \int_M \langle\nabla_{\mathbb{S}}h(x),v^X(x,t)\rangle_x\,d\nu_t(x).
    \label{eq:app-weak-nu}
\end{align}
The interchange of \(d/dt\) and the integral is justified by dominated
convergence, as \(\Phi\) is \(C^1\) and \(M\) is compact.

\smallskip
\noindent\emph{Step 4 (Equality via the transition flow).}
Fix $t \in (0, 1]$. To transport the test function backward from time $t$ to $s \in [0, t]$, we define
\[
    h_s := h \circ \Phi_t \circ \Phi_s^{-1}.
\]
Note that \(h_s \in C^1(M)\), \(h_t = h\), and \(h_0 = h \circ \Phi_t\). By definition, \(h_s(\Phi_s(x)) = h(\Phi_t(x))\), which is identically constant in \(s\). Differentiating this relation with respect to \(s\) using the chain rule gives
\[
    \partial_s h_s(\Phi_s(x)) + \langle \nabla_{\mathbb{S}}h_s(\Phi_s(x)), \partial_s \Phi_s(x) \rangle_{\Phi_s(x)} = 0.
\]
Since \(\partial_s \Phi_s(x) = v^X(\Phi_s(x), s)\), evaluating at \(y = \Phi_s(x)\) yields the backward transport equation:
\begin{equation}
    \partial_s h_s(y) = - \langle \nabla_{\mathbb{S}}h_s(y), v^X(y,s) \rangle_y.
    \label{eq:app-transport-of-h}
\end{equation}

Now, consider the difference \(F_h(s) := \int_M h_s\,d\mu_s - \int_M h_s\,d\nu_s\). Applying the weak formulations \eqref{eq:app-weak-mu} and \eqref{eq:app-weak-nu} extended to time-dependent test functions, we obtain:
\begin{align*}
    \tfrac{d}{ds}\!\int_M h_s\,d\mu_s
    &= \int_M \partial_s h_s(x)\,d\mu_s(x) + \int_M \langle\nabla_{\mathbb{S}}h_s(x),v^X(x,s)\rangle_x\,d\mu_s(x)\\
    &= \int_M \Big( -\langle\nabla_{\mathbb{S}}h_s(x),v^X(x,s)\rangle_x + \langle\nabla_{\mathbb{S}}h_s(x),v^X(x,s)\rangle_x \Big)\,d\mu_s(x) = 0,
\end{align*}
and identically \(\frac{d}{ds}\int_M h_s\,d\nu_s = 0\). Thus, \(F_h(s)\) is constant for \(s \in [0,t]\).

Evaluating at \(s=0\), since \(\mu_0 = \nu_0\), we have
\[
    F_h(0) = \int_M (h \circ \Phi_t)\,d\mu_0 - \int_M (h \circ \Phi_t)\,d\nu_0 = 0.
\]
Evaluating at \(s=t\), we get
\[
    F_h(t) = \int_M h\,d\mu_t - \int_M h\,d\nu_t = 0.
\]
Since this holds for every \(h \in C^\infty(M)\), we conclude \(\mu_t = \nu_t\).
\end{proof}

\subsection{Spherical Transport-Cost Contraction}
\label{sec::Spherical Trans}
\begin{definition}[Spherically rectifiable coupling]
\label{def:app-spherical-rectifiable-coupling}
A coupling \((X_0,X_1)\) on \(M\) is called spherically rectifiable if
\(X_0\neq -X_1\) almost surely and its spherical geodesic interpolation process
\[
    X=\left\{
        \exp_{X_0}\!\bigl(t\log_{X_0}(X_1)\bigr):t\in[0,1]
    \right\}
\]
is rectifiable.  In this case, the process \(Z=\{Z_t:t\in[0,1]\}\) in
\Cref{eq:app-spherical-ode} is called the spherical rectified flow of the
coupling \((X_0,X_1)\), denoted as
\(Z=\texttt{RectFlow}_{\mathbb{S}}((X_0,X_1))\), and \((Z_0,Z_1)\) is called
the spherical rectified coupling of \((X_0,X_1)\), denoted as
\((Z_0,Z_1)=\texttt{RectFlow}_{\mathbb{S}}((X_0,X_1))\).
\end{definition}

The Euclidean rectified-flow theorem controls arbitrary convex costs of the
displacement vector \(X_1-X_0\).  On a curved manifold there is no globally
consistent displacement vector, so the coordinate-free replacement is a convex
nondecreasing function of the geodesic distance.

\begin{theorem}[Contraction of convex distance costs]
\label{thm:app-distance-cost}
Let \((X_0,X_1)\) be spherically rectifiable and let
\((Z_0,Z_1)=\texttt{RectFlow}_{\mathbb{S}}((X_0,X_1))\).  For every convex nondecreasing
\(\phi:[0,\infty)\to\mathbb{R}\),
\begin{equation}
    \mathbb{E}\!\left[\phi\!\left(d_{\mathbb{S}}(Z_0,Z_1)\right)\right]
    \le
    \mathbb{E}\!\left[\phi\!\left(d_{\mathbb{S}}(X_0,X_1)\right)\right].
    \label{eq:app-distance-cost}
\end{equation}
In particular,
\[
    \mathbb{E}\!\left[d_{\mathbb{S}}(Z_0,Z_1)^2\right]
    \le
    \mathbb{E}\!\left[d_{\mathbb{S}}(X_0,X_1)^2\right] .
\]
\end{theorem}

\begin{proof}
For each ODE trajectory, \Cref{lem:app-length-energy} gives
\[
    d_{\mathbb{S}}(Z_0,Z_1)
    \le
    \int_0^1 \|\dot Z_t\|_{Z_t}\,dt
    =
    \int_0^1 \|v^X(Z_t,t)\|_{Z_t}\,dt .
\]
For each fixed \(t\), \Cref{lem:app-radial-convex} shows that
\(w\mapsto \phi(\|w\|)\) is convex on the tangent space \(T_{X_t}M\).  Using
\(v^X(X_t,t)=\mathbb{E}[\dot X_t\mid X_t]\) and Jensen's inequality in the
tangent space,
\[
\begin{aligned}
    \phi\!\left(\|v^X(X_t,t)\|_{X_t}\right)
    &=
    \phi\!\left(\|\mathbb{E}[\dot X_t\mid X_t]\|_{X_t}\right) \\
    &\le
    \mathbb{E}\!\left[\phi(\|\dot X_t\|_{X_t})\mid X_t\right].
\end{aligned}
\]
Therefore,
\begin{align*}
    \mathbb{E}\!\left[\phi(d_{\mathbb{S}}(Z_0,Z_1))\right]
    &\le
    \mathbb{E}\!\left[
        \phi\!\left(\int_0^1 \|v^X(Z_t,t)\|_{Z_t}\,dt\right)
    \right]
    &&\text{(by \Cref{lem:app-length-energy} and monotonicity of \(\phi\))}
    \\
    &\le
    \int_0^1
    \mathbb{E}\!\left[
        \phi\!\left(\|v^X(Z_t,t)\|_{Z_t}\right)
    \right]\,dt
    &&\text{(by Jensen's inequality in time)}
    \\
    &=
    \int_0^1
    \mathbb{E}\!\left[
        \phi\!\left(\|v^X(X_t,t)\|_{X_t}\right)
    \right]\,dt
    &&\text{(by \Cref{thm:app-marginal-preservation})}
    \\
    &\le
    \int_0^1
    \mathbb{E}\!\left[
        \mathbb{E}\!\left[\phi(\|\dot X_t\|_{X_t})\mid X_t\right]
    \right]\,dt
    &&\text{(by conditional Jensen's inequality)}
    \\
    &=
    \int_0^1
    \mathbb{E}\!\left[\phi(\|\dot X_t\|_{X_t})\right]\,dt
    &&\text{(by the law of total expectation)}
    \\
    &=
    \int_0^1
    \mathbb{E}\!\left[\phi(d_{\mathbb{S}}(X_0,X_1))\right]\,dt
    &&\text{(by \Cref{lem:app-geodesic})}
    \\
    &=
    \mathbb{E}\!\left[\phi(d_{\mathbb{S}}(X_0,X_1))\right]
    &&\text{(since the interval has length one)}
\end{align*}
which is \Cref{eq:app-distance-cost}.
\end{proof}

\subsection{The Straightening Effect}
\label{app::Straightening}

A spherically rectifiable coupling \((X_0,X_1)\) is called straight if it is a
fixed point of spherical rectification, that is, if
\((Z_0,Z_1)=\texttt{RectFlow}_{\mathbb{S}}((X_0,X_1))=(X_0,X_1)\).

For a spherically rectifiable coupling \((X_0,X_1)\), let
\[
    X_t=\exp_{X_0}\!\bigl(t\log_{X_0}(X_1)\bigr), \qquad t\in[0,1],
\]
be its spherical geodesic interpolation.  Define
\[
    V_{\mathbb{S}}((X_0,X_1))
    =
    \int_0^1
    \mathbb{E}\!\left[
        \|\dot X_t-v^X(X_t,t)\|_{X_t}^2
    \right]dt .
\]
For \(Z=\texttt{RectFlow}_{\mathbb{S}}((X_0,X_1))\), define
\[
    S_{\mathbb{S}}(Z)
    =
    \int_0^1
    \mathbb{E}\!\left[\|\dot Z_t\|_{Z_t}^2\right]dt
    -
    \mathbb{E}\!\left[d_{\mathbb{S}}(Z_0,Z_1)^2\right].
\]

\begin{theorem}[Spherical straight couplings]
\label{thm:app-straight-coupling}
Assume \((X_0,X_1)\) is spherically rectifiable, let \(X_t\) be its spherical
geodesic interpolation, and let
\(Z=\texttt{RectFlow}_{\mathbb{S}}((X_0,X_1))\).  Then \((X_0,X_1)\) is straight
iff the following equivalent statements hold:
\begin{enumerate}
    \item The squared spherical transport cost is unchanged:
    \[
        \mathbb{E}\!\left[d_{\mathbb{S}}(Z_0,Z_1)^2\right]
        =
        \mathbb{E}\!\left[d_{\mathbb{S}}(X_0,X_1)^2\right].
    \]
    \item The coupling is a fixed point of spherical rectification:
    \[
        (Z_0,Z_1)=(X_0,X_1).
    \]
    \item The rectified flow coincides with the spherical geodesic interpolation:
    \(Z_t=X_t\) almost surely for all \(t\in[0,1]\).
    \item \(V_{\mathbb{S}}((X_0,X_1))=0\).
\end{enumerate}
\end{theorem}

\begin{proof}
The implications \(3\Rightarrow2\Rightarrow1\) are immediate.  To show
\(1\Rightarrow4\), take \(\phi(r)=r^2\) in the proof of
\Cref{thm:app-distance-cost}.  The same calculation gives
\[
\begin{aligned}
    \mathbb{E}\!\left[d_{\mathbb{S}}(X_0,X_1)^2\right]
    &=
    \int_0^1
    \mathbb{E}\!\left[\|\dot X_t\|_{X_t}^2\right]dt \\
    &=
    V_{\mathbb{S}}((X_0,X_1))
    +
    \int_0^1
    \mathbb{E}\!\left[\|v^X(X_t,t)\|_{X_t}^2\right]dt \\
    &=
    V_{\mathbb{S}}((X_0,X_1))
    +
    \int_0^1
    \mathbb{E}\!\left[\|\dot Z_t\|_{Z_t}^2\right]dt ,
\end{aligned}
\]
Here the first equality uses the constant speed of the spherical geodesic in
\Cref{lem:app-geodesic}, the second equality uses the conditional variance
decomposition in \Cref{lem:app-conditional-tangent}, and the last equality uses
marginal preservation.  Therefore
\begin{equation}
    \mathbb{E}\!\left[d_{\mathbb{S}}(X_0,X_1)^2\right]
    -
    \mathbb{E}\!\left[d_{\mathbb{S}}(Z_0,Z_1)^2\right]
    =
    S_{\mathbb{S}}(Z)+V_{\mathbb{S}}((X_0,X_1)).
    \label{eq:app-straightening-identity}
\end{equation}
Both terms on the right are nonnegative: \(V_{\mathbb{S}}\ge0\) by definition,
and \(S_{\mathbb{S}}(Z)\ge0\) by \Cref{lem:app-length-energy}.  Therefore
equality of the squared costs implies \(V_{\mathbb{S}}((X_0,X_1))=0\).

It remains to show \(4\Rightarrow3\).  If \(V_{\mathbb{S}}((X_0,X_1))=0\), then
\[
    \dot X_t=v^X(X_t,t) \qquad \text{a.s. for a.e. } t.
\]
By Fubini's theorem, after removing a single probability-zero set of sample
paths, this equality holds for Lebesgue-a.e. \(t\) along each remaining path.
Since the paths are path-wise continuously differentiable, the fundamental
theorem of calculus in the ambient Euclidean space gives, for every
\(t\in[0,1]\),
\[
    X_t=X_0+\int_0^t v^X(X_s,s)\,ds .
\]
Thus \(X\) satisfies the same characteristic equation as \(Z\) with the same
initial condition.  By the uniqueness assumed in spherical rectifiability,
\(X=Z\).
\end{proof}

\begin{theorem}[Spherical reflow straightening rate]
\label{thm:app-reflow-rate}
Let \((X_0,X_1)\) be a coupling on \(M\), let
\(N_{\text{reflow}}\ge1\) be an integer, and set
\((Z_0^0,Z_1^0)=(X_0,X_1)\).  Given \((Z_0^k,Z_1^k)\), let
\[
    Z^{k+1}=\texttt{RectFlow}_{\mathbb{S}}((Z_0^k,Z_1^k)),
\]
and write \((Z_0^{k+1},Z_1^{k+1})\) for its endpoint coupling.
Assume each \((Z_0^k,Z_1^k)\) is spherically rectifiable for
\(k=0,\ldots,N_{\text{reflow}}-1\).  Then
\begin{equation}
    \sum_{k=0}^{N_{\text{reflow}}-1}
    \left[
        S_{\mathbb{S}}(Z^{k+1})
        +
        V_{\mathbb{S}}((Z_0^k,Z_1^k))
    \right]
    \le
    \mathbb{E}\!\left[d_{\mathbb{S}}(X_0,X_1)^2\right].
    \label{eq:app-reflow-sum}
\end{equation}
Consequently, if
\(\mathbb{E}[d_{\mathbb{S}}(X_0,X_1)^2]<+\infty\), then
\[
    \min_{0\le k\le N_{\text{reflow}}-1}
    \left[
        S_{\mathbb{S}}(Z^{k+1})
        +
        V_{\mathbb{S}}((Z_0^k,Z_1^k))
    \right]
    =
    O(1/N_{\text{reflow}}).
\]
\end{theorem}

\begin{proof}
Apply \Cref{eq:app-straightening-identity} to each rectification step:
\[
    \mathbb{E}\!\left[d_{\mathbb{S}}(Z_0^k,Z_1^k)^2\right]
    -
    \mathbb{E}\!\left[d_{\mathbb{S}}(Z_0^{k+1},Z_1^{k+1})^2\right]
    =
    S_{\mathbb{S}}(Z^{k+1})
    +
    V_{\mathbb{S}}((Z_0^k,Z_1^k)).
\]
Summing over \(k=0,\ldots,N_{\text{reflow}}-1\) telescopes and yields
\[
    \sum_{k=0}^{N_{\text{reflow}}-1}
    \left[
        S_{\mathbb{S}}(Z^{k+1})
        +
        V_{\mathbb{S}}((Z_0^k,Z_1^k))
    \right]
    =
    \mathbb{E}\!\left[d_{\mathbb{S}}(Z_0^0,Z_1^0)^2\right]
    -
    \mathbb{E}\!\left[d_{\mathbb{S}}(Z_0^{N_{\text{reflow}}},
    Z_1^{N_{\text{reflow}}})^2\right].
\]
Since the second term on the right is nonnegative and
\((Z_0^0,Z_1^0)=(X_0,X_1)\), \Cref{eq:app-reflow-sum} follows.  The
\(O(1/N_{\text{reflow}})\) statement follows from the explicit bound
\[
    \min_{0\le k\le N_{\text{reflow}}-1}
    \left[
        S_{\mathbb{S}}(Z^{k+1})
        +
        V_{\mathbb{S}}((Z_0^k,Z_1^k))
    \right]
    \le
    \frac{
        \mathbb{E}\!\left[d_{\mathbb{S}}(X_0,X_1)^2\right]
    }{N_{\text{reflow}}}.
\]
\end{proof}

\

%% file: checklist.tex
\section*{NeurIPS Paper Checklist}

\begin{enumerate}

\item {\bf Claims}
    \item[] Question: Do the main claims made in the abstract and introduction accurately reflect the paper's contributions and scope?
    \item[] Answer: \answerYes{} 
    \item[] Justification: The main claims in the abstract and introduction are aligned with the proposed method, theoretical motivation, and empirical results. We clearly state the scope of the work, including that the proposed rectification strategy is designed for categorical distributions on the probability simplex and is evaluated on the considered datasets.

\item {\bf Limitations}
    \item[] Question: Does the paper discuss the limitations of the work performed by the authors?
    \item[] Answer: \answerYes{}
    \item[] Justification: The paper discusses the limitations of the proposed manifold-based formulation in the ``Limitations and Broader Impact'' section. In particular, it notes that the Fisher--Rao geometry treats categorical states symmetrically and does not incorporate pre-existing semantic priors, which may limit performance in tasks where hierarchical semantic relationships among tokens are important.

\item {\bf Theory assumptions and proofs}
    \item[] Question: For each theoretical result, does the paper provide the full set of assumptions and a complete (and correct) proof?
    \item[] Answer: \answerYes{} 
    \item[] Justification: The paper states the assumptions required for the theoretical results and provides complete proofs in the appendix. The main text presents the key formulations and proof intuitions, while the formal derivations are given with the corresponding theorem and lemma statements properly numbered and cross-referenced.

    \item {\bf Experimental result reproducibility}
    \item[] Question: Does the paper fully disclose all the information needed to reproduce the main experimental results of the paper to the extent that it affects the main claims and/or conclusions of the paper (regardless of whether the code and data are provided or not)?
    \item[] Answer: \answerYes{} 
    \item[] Justification: The paper provides the information needed to reproduce the main experimental results, including the datasets, evaluation metrics, model configurations, training procedures, and baseline methods. Additional implementation and experimental details are provided in the appendix.

\item {\bf Open access to data and code}
    \item[] Question: Does the paper provide open access to the data and code, with sufficient instructions to faithfully reproduce the main experimental results, as described in supplemental material?
    \item[] Answer: \answerYes{} 
    \item[] Justification: The code has been submitted as a compressed archive in the supplementary material. We include instructions for reproducing the main experimental results, and the code will be made publicly available upon acceptance.

\item {\bf Experimental setting/details}
    \item[] Question: Does the paper specify all the training and test details (e.g., data splits, hyperparameters, how they were chosen, type of optimizer) necessary to understand the results?
    \item[] Answer: \answerYes{}
    \item[] Justification:The paper specifies the experimental settings, including datasets, data preprocessing, train/test protocols, model architectures, hyperparameters, optimizers, and evaluation metrics. These details are provided in the experimental section and appendix.

\item {\bf Experiment statistical significance}
    \item[] Question: Does the paper report error bars suitably and correctly defined or other appropriate information about the statistical significance of the experiments?
    \item[] Answer: \answerYes{} 
    \item[] Justification: We report error bars for the main experiments where repeated runs are computationally feasible. For experiments without explicit error bars, especially sample-based evaluations, we reduce estimation variance by increasing the number of generated samples and apply the same evaluation protocol across all methods.

\item {\bf Experiments compute resources}
    \item[] Question: For each experiment, does the paper provide sufficient information on the computer resources (type of compute workers, memory, time of execution) needed to reproduce the experiments?
    \item[] Answer: \answerYes{} 
    \item[] Justification: The paper reports the compute resources used for the experiments, including the type of hardware, GPU memory, and approximate training or evaluation time. These details are provided to facilitate reproducibility of the reported results.
    
\item {\bf Code of ethics}
    \item[] Question: Does the research conducted in the paper conform, in every respect, with the NeurIPS Code of Ethics \url{https://neurips.cc/public/EthicsGuidelines}?
    \item[] Answer: \answerYes{}
    \item[] Justification: The research conforms to the NeurIPS Code of Ethics. The work is methodological in nature, uses standard benchmark datasets, and does not involve personally identifiable information, human-subject experiments, or unsafe data collection.

\item {\bf Broader impacts}
    \item[] Question: Does the paper discuss both potential positive societal impacts and negative societal impacts of the work performed?
    \item[] Answer: \answerYes{}
    \item[] Justification: The paper discusses both positive and negative broader impacts in the ``Limitations and Broader Impact'' section. It highlights potential benefits for efficient discrete generative modeling in language, biology, and scientific domains, while also noting possible misuse risks associated with improved generative models, such as disinformation or malicious content generation.
    
\item {\bf Safeguards}
    \item[] Question: Does the paper describe safeguards that have been put in place for responsible release of data or models that have a high risk for misuse (e.g., pre-trained language models, image generators, or scraped datasets)?
    \item[] Answer: \answerNA{}
    \item[] Justification: The paper does not release high-risk pretrained generative models, scraped datasets, or deployment-ready systems that would require additional safeguards. The contribution is a foundational methodological framework evaluated on standard categorical datasets.

\item {\bf Licenses for existing assets}
    \item[] Question: Are the creators or original owners of assets (e.g., code, data, models), used in the paper, properly credited and are the license and terms of use explicitly mentioned and properly respected?
    \item[] Answer: \answerYes{} 
    \item[] Justification: The paper properly credits the existing datasets, codebases, and software libraries used in the experiments. The corresponding references, versions, and license information are provided where available.

\item {\bf New assets}
    \item[] Question: Are new assets introduced in the paper well documented and is the documentation provided alongside the assets?
    \item[] Answer: \answerYes{}
    \item[] Justification: The paper releases the implementation of the proposed method as a new code asset, together with documentation and instructions for reproducing the main experimental results.

\item {\bf Crowdsourcing and research with human subjects}
    \item[] Question: For crowdsourcing experiments and research with human subjects, does the paper include the full text of instructions given to participants and screenshots, if applicable, as well as details about compensation (if any)? 
    \item[] Answer: \answerNA{} 
    \item[] Justification: The paper does not involve crowdsourcing experiments or research with human subjects.

\item {\bf Institutional review board (IRB) approvals or equivalent for research with human subjects}
    \item[] Question: Does the paper describe potential risks incurred by study participants, whether such risks were disclosed to the subjects, and whether Institutional Review Board (IRB) approvals (or an equivalent approval/review based on the requirements of your country or institution) were obtained?
    \item[] Answer: \answerNA{}
    \item[] Justification: The paper does not involve crowdsourcing or human-subject research, and therefore IRB approval or equivalent review is not applicable.

\item {\bf Declaration of LLM usage}
    \item[] Question: Does the paper describe the usage of LLMs if it is an important, original, or non-standard component of the core methods in this research? Note that if the LLM is used only for writing, editing, or formatting purposes and does \emph{not} impact the core methodology, scientific rigor, or originality of the research, declaration is not required.
    \item[] Answer: \answerNA{} 
    \item[] Justification: The core method development in this paper does not involve LLMs as an important, original, or non-standard component. Any use of LLMs, if any, was limited to writing, editing, or formatting assistance and did not affect the methodology, experiments, or scientific claims.

\end{enumerate}